%% file: spectral-slda.tex
\documentclass[10pt,journal,cspaper,compsoc]{IEEEtran}

\ifCLASSOPTIONcompsoc
\else
\fi

\ifCLASSINFOpdf
\else
\fi

\usepackage{amsmath,amsfonts,amsthm,amssymb}
\usepackage{algorithm}
\usepackage{algpseudocode}
\usepackage{footnote}
\usepackage{graphicx,subfigure}
\usepackage{tikz}
\usepackage{verbatim}
\usetikzlibrary{fit,positioning}
\usepackage{multirow}
\usepackage{threeparttable}
\usepackage{bm}

\usepackage{color}
\usepackage{bm}
\allowdisplaybreaks[3]
\newcommand{\vct}{\bm}

\newcommand{\hv}{\mathbf{h}}

\newcommand{\diag}{\mathrm{diag}}

\newcommand{\etav}{\boldsymbol \eta}

\newtheorem{theorem}{{\bf Theorem}}

\newtheorem{thm}{Theorem}
\newtheorem{lem}{Lemma}
\newtheorem{prop}{Proposition}
\newtheorem{cor}{Corollary}
\newtheorem{deff}{Definition}
\newtheorem{rem}{Remark}

\newtheorem*{prof*}{Proof}

\def\QEDopen{{\setlength{\fboxsep}{0pt}\setlength{\fboxrule}{0.2pt}\fbox{\rule[0pt]{0pt}{1.3ex}\rule[0pt]{1.3ex}{0pt}}}}
\def\QED{\QEDopen} 

\newcommand{\drawbox}[4]{
\draw[fill=#4] ({#1-1+(#2-1)*2/3*cos(30)},{(#2-1)*2/3*sin(30)+#3-1}) -- ({#1-1+(#2-1)*2/3*cos(30)+1},{(#2-1)*2/3*sin(30)+#3-1}) -- ({#1-1+(#2-1)*2/3*cos(30)+1},{(#2-1)*2/3*sin(30)+#3-1+1}) -- ({#1-1+(#2-1)*2/3*cos(30)},{(#2-1)*2/3*sin(30)+#3-1+1}) -- cycle;
\draw[fill=#4] ({#1-1+(#2-1)*2/3*cos(30)},{(#2-1)*2/3*sin(30)+#3-1+1}) -- ({#1-1+#2*2/3*cos(30)},{(#2-1)*2/3*sin(30)+#3-1+1+2/3*sin(30)}) -- ({#1-1+(#2-1)*2/3*cos(30)+1+2/3*cos(30)},{(#2-1)*2/3*sin(30)+#3-1+1+2/3*sin(30)}) -- ({#1-1+(#2-1)*2/3*cos(30)+1},{(#2-1)*2/3*sin(30)+#3-1+1}) -- cycle;
\draw[fill=#4] ({#1-1+(#2-1)*2/3*cos(30)+1},{(#2-1)*2/3*sin(30)+#3-1+1}) -- ({#1-1+(#2-1)*2/3*cos(30)+1+2/3*cos(30)},{(#2-1)*2/3*sin(30)+#3-1+1+2/3*sin(30)}) -- ({#1-1+(#2-1)*2/3*cos(30)+1+2/3*cos(30)},{(#2-1)*2/3*sin(30)+#3-1+2/3*sin(30)}) -- ({#1-1+(#2-1)*2/3*cos(30)+1},{(#2-1)*2/3*sin(30)+#3-1}) -- cycle;}

\usepackage{enumerate}
\usepackage{epsf}
\usepackage{multirow}
\usepackage{wrapfig}
\usepackage{graphicx}
\usepackage{subfigure} 

\usepackage{color}

\newcommand{\jun}[1]{{\color{red}{\bf\sf [#1]}}}

\begin{document}
%
\title{Spectral Learning for Supervised \\Topic Models}

\author{Yong Ren$^\ddag$, Yining Wang$^\ddag$, Jun Zhu,~\IEEEmembership{Member,~IEEE}
\IEEEcompsocitemizethanks{\IEEEcompsocthanksitem $^\ddag$Y.R. and Y.W. contributed equally.
\IEEEcompsocthanksitem Y. Ren and J. Zhu are with Department of Computer Science and Technology; TNList Lab; State Key Laboratory for Intelligent Technology and Systems; Center for Bio-Inspired Computing Research, Tsinghua University, Beijing, 100084 China. 
Email: reny11@mails.tsinghua.edu.cn;~dcszj@tsinghua.edu.cn.
\IEEEcompsocthanksitem Y. Wang is with Machine Learning Department, Carnegie Mellon University, Pittsburgh, USA. 
Email: yiningwa@andrew.cmu.edu}
\thanks{}}

\markboth{Journal of \LaTeX\ Class Files,~Vol.~6, No.~1, January~2007}%
{Shell \MakeLowercase{\textit{et al.}}: Bare Demo of IEEEtran.cls for Computer Society Journals}
%

\IEEEcompsoctitleabstractindextext{%
\begin{abstract}
Supervised topic models simultaneously model the latent topic structure of large collections of documents and a response variable associated with each document. Existing inference methods are based on variational approximation or Monte Carlo sampling, which often suffers from the local minimum defect. Spectral methods have been applied to learn unsupervised topic models, such as latent Dirichlet allocation (LDA), with provable guarantees. This paper investigates the possibility of applying spectral methods to recover the parameters of supervised LDA (sLDA). We first present a two-stage spectral method, which recovers the parameters of LDA followed by a power update method to recover the regression model parameters. 
Then, we further present a single-phase spectral algorithm to jointly recover the topic distribution matrix as well as the regression weights. Our spectral algorithms are provably correct and computationally efficient.
We prove a sample complexity bound for each algorithm and subsequently derive a sufficient condition for the identifiability of sLDA.
Thorough experiments on synthetic and real-world datasets verify the theory and demonstrate the practical effectiveness
of the spectral algorithms. In fact, our results on a large-scale review rating dataset demonstrate
that our single-phase spectral algorithm alone gets comparable or even better performance than
state-of-the-art methods, while previous work on spectral methods has rarely
reported such promising performance.

\end{abstract}

}

\maketitle

\IEEEdisplaynotcompsoctitleabstractindextext

%
\IEEEpeerreviewmaketitle

\section{Introduction}

%
%
%
\IEEEPARstart{T}{opic} modeling offers a suite of useful tools that automatically learn the latent semantic structure of a large collection of documents or images, with 
latent Dirichlet allocation (LDA)~\cite{lda} as one of the most popular examples. 
The vanilla LDA is an unsupervised model built on input contents of documents or images. In many applications side information is often available apart from raw contents, e.g., user-provided rating scores of an online review text or user-generated tags for an image. Such side signal usually provides additional information to reveal the underlying structures of the data in study. There have been extensive studies on developing topic models that incorporate various side information, e.g., by treating it as supervision. Some representative models are supervised LDA (sLDA) \cite{slda} that captures a real-valued regression response for each document, multiclass sLDA \cite{multislda} that learns with discrete classification responses, discriminative LDA (DiscLDA) \cite{disclda} that incorporates classification response via discriminative linear transformations on topic mixing vectors, and MedLDA \cite{medlda,gibbsmedlda} that employs a max-margin criterion to learn discriminative latent topic representations for accurate prediction.

Topic models are typically learned by finding maximum likelihood estimates (MLE) through local search or sampling methods \cite{variationallda,gibbslda,emlda}, 
which may get trapped in local optima. Much recent progress has been made on developing spectral decomposition~\cite{a:twosvd,a:tensordecomp,spechmm} and nonnegative matrix factorization (NMF) \cite{practicalalgo,nmf1,beyondsvd,nmf2} methods to estimate the topic-word distributions. Instead of finding MLE estimates, which is a known NP-hard problem \cite{beyondsvd}, these methods assume that the documents are i.i.d. sampled from a topic model, and attempt to recover the underlying model parameters. Compared to local search and sampling algorithms, these methods enjoy the advantage of being provably effective. In fact, sample complexity bounds have been proved to show that given a sufficiently large collection of documents, these algorithms can recover the model parameters accurately with a high probability.

Recently, some attention has been paid on supervised topic models with NMF methods. For example, Nguyen et al.~\cite{a:sAnchor} present an extension of the anchor-word methods~\cite{practicalalgo} for LDA to capture categorical information in a supervised LDA for sentiment classification.
However, for spectral methods, previous work has mainly focused on unsupervised latent variable models, leaving the broad family of supervised models (e.g., sLDA) largely unexplored.
The only exception
is~\cite{specregression} which presents a spectral method for mixtures of regression models, quite different from sLDA.
Such ignorance is not a coincidence as supervised models impose new technical challenges. For instance, a direct application of previous techniques~\cite{a:twosvd,a:tensordecomp} on sLDA
cannot handle regression models with duplicate entries.
In addition, the sample complexity bound gets much worse if we try to match entries in regression models
with their corresponding topic vectors. 

In this paper, we extend the applicability of spectral learning methods by presenting novel spectral decomposition algorithms to recover the parameters of sLDA models from low-order empirical moments estimated from the data. We present two variants of spectral methods.
The first algorithm is an extension of the spectral methods for LDA,
with an extra power update step of recovering the regression model in sLDA, including the variance parameter.
The power-update step uses a newly designed empirical moment to recover regression model entries directly from the data and reconstructed topic distributions.
It is free from making any constraints on the underlying regression model.
We provide a sample complexity bound and analyze the identifiability conditions.
In fact, the two-stage method does not increase the sample complexity much compared to that of the vanilla LDA.

However, the two-stage algorithm could have some disadvantages because of its separation that the topic distribution matrix is recovered in an unsupervised manner without considering
supervision and the regression parameters are recovered by assuming a fixed topic matrix.
Such an unwarranted separation often leads to inferior performance compared to Gibbs sampling methods in practice (See Section~\ref{sec:exp}).
To address this problem, we further present a novel single-phase spectral method for supervised topic models,
which jointly recovers all the model parameters (except the noise variance) by doing a single-step of robust tensor decomposition of a newly designed
empirical moment that takes both input data and supervision signal into consideration.
Therefore, the joint method can use supervision information in recovering both the topic distribution matrix and
regression parameters. The joint method is also provably correct and we provide a sample complexity bound to achieve the $\epsilon$
error rate in a high probability.

Finally, we provide thorough experiments on both synthetic and real datasets to demonstrate the practical effectiveness of our spectral methods.
For the two-stage method, by combining with a Gibbs sampling procedure, we show superior performance in terms of language modeling, prediction accuracy and running time
compared to traditional inference algorithms.
Furthermore, we demonstrate that on a large-scale review rating dataset
our single-phase method alone can achieve comparable or even better results than the state-of-the-art methods (e.g., sLDA with Gibbs sampling and MedLDA).
Such promising results are significant to the literature of spectral methods, which were often observed to
be inferior to the MLE-based methods
; and a common heuristic was to use the outputs
of a spectral method to initialize an EM algorithm, which sometimes improves the performance~\cite{a:meetem}.

The rest of the paper is organized as follows. Section $2$ reviews basics of supervised topic models. Section $3$ introduces the background knowledge and notations for sLDA and high-order tensor decomposition.
Section $4$ presents the two-stage spectral method, with a rigorous theoretical analysis, and Section $5$ presents the joint spectral method together with a sample complexity bound.
Section $6$ presents some implementation details to scale up the computation.
Section $7$ presents experimental results on both synthetic and real datasets.
Finally, we conclude in Section $8$. 

\section{Related Work}

Based on different principles, there are various methods to learn supervised topic models. The most natural one is maximum-likelihood estimation (MLE). However, the
highly non-convex property of the learning objective makes the optimization problem very hard. In the original paper \cite{slda}, where the MLE is used, the authors choose variational
approximation to handle intractable posterior expectations. Such a method tries to maximize a lower bound which is built on some variational distributions and a mean-field assumption is usually imposed for tractability. Although the method works well in practice, we do not have any guarantee that the distribution we learned is close to the true one. Under Bayesian framework, Gibbs sampling is an attractive method which enjoys the property that the stationary distribution of the chain is the target posterior distribution. However, this does not mean that we can really get accurate samples from posterior distribution in practice. The slow mixing rate often makes the sampler trapped in a local minimum which is far from the
true distribution if we only run a finite number of iterations.

Max-margin learning is another principle on learning supervised topic models, with maximum entropy discrimination LDA (MedLDA)~\cite{medlda} as a popular example. MedLDA explores the max-margin principle to learn sparse and discriminative topic representations. The learning problem can be defined under the regularized Bayesian inference (RegBayes)~\cite{regBayes} framework, where the max-margin posterior regularization is introduced to ensure that the topic representations are good at predicting response variables.
Though both carefully designed variational inference~\cite{medlda} and Gibbs sampling methods~\cite{gibbsmedlda} are given, we still cannot guarantee the quality of the learnt model in general.

Recently, increasing efforts have been made to recover the parameters directly with provable correctness for topic models, with the main focus on unsupervised models such as LDA.
Such methods adopt either NMF or spectral decomposition approaches.
For NMF, the basic idea is that even NMF is an NP-hard problem in general, the topic distribution matrix can be recovered under
some separable condition, e.g., each topic has at least one anchor word.
Precisely, for each topic, the method first finds an anchor word that has non-zero probability only in that topic. 
Then a recovery step reconstructs the topic distribution given such
anchor words and a second-order moment matrix of word-word co-occurrence~\cite{beyondsvd,practicalalgo}.
The original reconstruction step only needs a part of the matrix which is not robust in practice.
Thus in \cite{practicalalgo}, the author recovers the topic distribution based on a probabilistic framework.
The NMF methods produce good empirical results on real-world data.
Recently, the work \cite{a:sAnchor} extends the anchor word methods to handle supervised topic models. The method
augments the word co-occurrence matrix with additional dimensions for metadata such as sentiment, and 
shows better performance in sentiment classification.

Spectral methods start from computing some low-order moments based on the samples and then relate them with the model parameters.
For LDA, tensors up to three order are sufficient to recover its parameters~\cite{a:twosvd}.
After centralization, the moments can be
expressed as a mixture of the parameters we are interested in.
After that whitening and robust tensor decomposition steps are adopted
to recover model parameters. The whitening step makes that the third-order tensor can be decomposed as a set of
orthogonal eigenvectors and their corresponding eigenvalues after some operations based on its output
and the robust tensor decomposition step then finds them.
Previous work only focuses on unsupervised LDA models and
we aim to extend the ability for spectral methods to handle response variables. Finally,
some preliminary results of the two-stage recovery algorithm have been reported in~\cite{a:yiningwang}.
This paper presents a systematical analysis with a novel one-stage spectral method,
which yields promising results on a large-scale dataset.

\vspace{-.15cm}
\section{Preliminaries}

We first overview the basics of sLDA, orthogonal tensor decomposition and the notations to be used.

\vspace{-.15cm}
\subsection{Supervised LDA}\label{sec:sLDA-model}

Latent Dirichlet allocation (LDA)~\cite{lda} is a hierarchical generative model for topic modeling of text documents or images represented in a bag-of-visual-words format~\cite{ldaimage}. It assumes $k$ different \emph{topics} with topic-word distributions $\vct\mu_1,\cdots,\vct\mu_k\in\Delta^{V-1}$, where $V$ is the vocabulary size and $\Delta^{V-1}$ denotes the probability simplex of a $V$-dimensional random vector.
For a document, LDA models a 
topic mixing vector $\vct h\in\Delta^{k-1}$
as a probability distribution over the $k$ topics. A conjugate Dirichlet prior with parameter $\vct \alpha$ is imposed on the topic mixing vectors. A bag-of-words model is then adopted, which first generates a topic indicator for each word $z_j \sim \textrm{Multi}(\vct h)$ and then generates the word itself as $w_j \sim \textrm{Multi}(\vct \mu_{z_j})$.
%
Supervised latent Dirichlet allocation (sLDA)~\cite{slda} incorporates an extra response variable $y\in\mathbb R$ for each document.
The response variable is modeled by a linear regression model $\vct\eta\in\mathbb R^k$
on either the topic mixing vector $\vct h$ or the averaging topic assignment vector $\bar{\vct z}$,
where $\bar z_i = \frac{1}{M}\sum_{j=1}^M {1_{[z_j=i]}}$ with $M$ being the number of words in the document and $1_{[\cdot]}$ being the indicator function (i.e., equals to 1 if the predicate holds; otherwise 0).
The noise is assumed to be Gaussian with zero mean and $\sigma^2$ variance.

Fig.~\ref{slda} shows the graph structure of the sLDA model using $\vct h$ for regression.
Although previous work has mainly focused on the model using averaging topic assignment vector $\bar{\vct z}$, which is convenient for collapsed Gibbs sampling and variational inference with $\vct h$ integrated out due to the conjugacy between a Dirichlet prior and a multinomial likelihood,
we consider using the topic mixing vector $\vct h$ as the features for regression because it will considerably simplify our spectral algorithm and analysis.
One may assume that whenever a document is not too short, the empirical distribution of its word topic assignments should be close to the document's topic mixing vector.
Such a scheme was adopted to learn sparse topic coding models~\cite{Zhu:stc11}, and has demonstrated promising results in practice.
Our results also prove that this is an effective strategy.

\iftrue
\begin{figure}[t]\vspace{-.1cm}
\centering
\begin{tikzpicture}
\tikzstyle{main}=[circle, minimum size = 5mm, thick, draw =black!80, node distance = 6mm]
\tikzstyle{connect}=[-latex, thick]
\tikzstyle{box}=[rectangle, draw=black!100]
  \node[main, fill = white!100] (alpha) [label=center:$\alpha$] { };
  \node[main] (h) [right=of alpha,label=center:h] { };
  \node[main] (z) [right=of h,label=center:z] {};
  \node[main] (mu) [above=of z,label=center:$\bm{\mu}$] { };
  \node[main, fill = black!10] (x) [right=of z,label=center:x] { };
  \node[main, fill = black!10] (y) [below= of h, label=center:y]{ };
  \node[main] (eta) [below=of alpha, label=center:$\bm{\eta}$]{};
  \node (exp) [below= of z]{$y = \bm{\eta}^{\top}\bm{h} + \epsilon $};
  \path (alpha) edge [connect] (h)
        (h) edge [connect] (z)
        (z) edge [connect] (x)
        (mu) edge [connect] (x)
		(h) edge [connect] (y)
		(eta) edge [connect] (y);
  \node[rectangle, inner sep=-0.9mm, fit= (z) (x),label=below right:M, xshift=5mm] {};
  \node[rectangle, inner sep=3.6mm,draw=black!100, fit= (z) (x)] {};
  \node[rectangle, inner sep=1mm, fit= (z) (x) (y) ,label=below right:N, xshift=5mm] {};
  \node[rectangle, inner sep=5mm, draw=black!100, fit = (h) (z) (x) (y) ] {};
\end{tikzpicture}\vspace{-.3cm}
\caption{A graphical illustration of supervised LDA. 
See text for details.}
\label{slda} \vspace{-.4cm}
\end{figure}
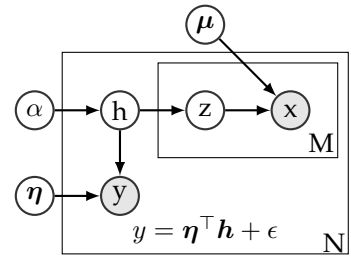
\fi


\vspace{-.15cm}
\subsection{High-order tensor product and orthogonal tensor decomposition}
Here we briefly introduce something about tensors,
which mostly follow the same as in \cite{a:tensordecomp}.
A real \emph{$p$-th order tensor} $A\in\bigotimes_{i=1}^p{\mathbb R^{n_i}}$ belongs to the tensor product of Euclidean spaces $\mathbb R^{n_i}$.
Without loss of generality, we assume $n_1=n_2=\cdots=n_p=n$. We can identify each coordinate of $A$ by a $p$-tuple
$(i_1,\cdots,i_p)$, where $i_1,\cdots,i_p\in [n]$.
For instance, a $p$-th order tensor is a vector when $p=1$ and a matrix when $p=2$. We can also consider a $p$-th order tensor $A$ as a multilinear mapping.
For $A\in\bigotimes^p\mathbb R^n$ and matrices $X_1,\cdots,X_p\in\mathbb R^{n\times m}$,
the mapping $A(X_1,\cdots,X_p)$ is a $p$-th order tensor in $\bigotimes^p\mathbb R^m$, with
$[A(X_1,\cdots,X_p)]_{i_1,\cdots,i_p}
 \triangleq \sum_{j_1,\cdots,j_p\in [n]}{A_{j_1,\cdots,j_p}[X_1]_{j_1,i_1}[X_2]_{j_2,i_2}\cdots [X_p]_{j_p,i_p}}$.
Consider some concrete examples of such a multilinear mapping. When $A$, $X_1$ and $X_2$ are matrices, we have $A(X_1,X_2) = X_1^\top AX_2$. Similarly, when $A$ is a matrix and $x$ is a vector, we have $A(I, x) = Ax$.

An orthogonal tensor decomposition of a tensor $A\in\bigotimes^p\mathbb R^n$ is a collection of orthonormal vectors $\{\vct v_i\}_{i=1}^k$
and scalars $\{\lambda_i\}_{i=1}^k$ such that
$A = \sum_{i=1}^k{\lambda_i\vct v_i^{\otimes p}}$, where we use $\vct v^{\otimes p} \triangleq \vct v\otimes \vct v\otimes\cdots\otimes \vct v$ to denote the $p$-th order tensor
generated by a vector $\vct v$.
Without loss of generality, we assume $\lambda_i$ are nonnegative when $p$ is odd since we can change the sign of $\bm{v_i}$ otherwise.
Although orthogonal tensor decomposition in the matrix case
can be done efficiently by singular value decomposition (SVD),
it has several delicate issues in higher order tensor spaces~\cite{a:tensordecomp}.
For instance, tensors may not have unique decompositions,
and an orthogonal decomposition may not exist for every symmetric tensor~\cite{a:tensordecomp}.
Such issues are further complicated when only noisy estimates of the desired tensors are available.
For these reasons, we need more advanced techniques to handle high-order tensors. In this paper, we will apply \emph{robust tensor power} methods~\cite{a:tensordecomp} to recover robust eigenvalues and eigenvectors of an (estimated) third-order tensor. The algorithm recovers eigenvalues and eigenvectors up to an absolute error $\varepsilon$,
while running in polynomial time w.r.t the tensor dimension and $\log(1/\varepsilon)$.
Further details and analysis of the robust tensor power method are in Appendix~A.2 and \cite{a:tensordecomp}.

\vspace{-.15cm}
\subsection{Notations}

We use $\|\vct v\| = \sqrt{\sum_i{v_i^2}}$ to denote the Euclidean norm of vector $\vct v$,
$\|M\|$ to denote the spectral norm of matrix $M$, $\|T\|$ to denote the operator norm of a high-order tensor, and
$\|M\|_F = \sqrt{\sum_{i,j}{M_{ij}^2}}$ to denote the Frobenious norm of $M$. We use an one-hot vector $\vct x\in\mathbb R^V$ to represent a word in a document (i.e., for the $i$-th word in a vocabulary, only $x_i = 1$ all other elements are $0$). 
In the two-stage spectral method, we use $O \triangleq (\vct\mu_1,\vct\mu_2,\cdots,\vct\mu_k)\in\mathbb R^{V\times k}$ to denote
the topic distribution matrix, and $\widetilde O\triangleq (\widetilde{\vct\mu}_1,\widetilde{\vct\mu}_2,\cdots,\widetilde{\vct\mu}_k)$
to denote the canonical version of $O$, where $\widetilde{\vct\mu}_i = \sqrt{\frac{\alpha_i}{\alpha_0(\alpha_0+1)}}\vct\mu$ with
$\alpha_0\triangleq \sum_{i=1}^k{\alpha_i}$. For the joint spectral method, we combine the topic distribution $\bm{\mu}_i$ with its
regression parameter to form a joint topic distribution vector $\bm{v}_i \triangleq [\bm{\mu}_i',\eta_i]'$. We use notation $O^* \triangleq (\bm{v}_1, \bm{v}_2, ... ,
\bm{v}_k) \in \mathbb{R}^{(V+1) \times k} $ to denote the joint topic distribution matrix and $\widetilde{O}^* \triangleq [\widetilde{\bm{v}}_1, \widetilde{\bm{v}}_2, ...,
\widetilde{\bm{v}}_k]$ to denote its canonical version where $\widetilde{\bm{v}}_i \triangleq \sqrt{\frac{\alpha_i}{\alpha_0(\alpha_0 + 1)}} \bm{v}_i$.

\vspace{-.15cm}
\section{A Two-stage Spectral Method}

We first present a two-stage spectral method to recover the parameters of sLDA. 
The algorithm consists of two key components---an orthogonal tensor decomposition of observable moments to recover the topic distribution matrix $O$ and a power update method to recover the regression model $\vct\eta$. We present these techniques and a rigorous theoretical analysis below. 

\vspace{-.15cm}
\subsection{Moments of observable variables}

Our spectral decomposition methods recover the topic distribution matrix $O$ and the linear regression model $\vct\eta$ by manipulating moments of observable variables. In Definition~\ref{def_moment}, we define a list of moments on random variables from the underlying sLDA model.
\begin{deff} We define the following moments of observable variables:
{\small \setlength\arraycolsep{1pt} \begin{eqnarray}
M_1 &=& \mathbb E[\vct x_1],\quad
M_2 = \mathbb E[\vct x_1\otimes\vct x_2] - \frac{\alpha_0}{\alpha_0+1}M_1\otimes M_1,\\
M_3 &=& \mathbb E[\vct x_1\otimes\vct x_2\otimes\vct x_3] - \frac{\alpha_0}{\alpha_0+2}(\mathbb E[\vct x_1\otimes\vct x_2\otimes M_1] \nonumber \\
& &+ \mathbb E[\vct x_1\otimes M_1\otimes \vct x_2] + \mathbb E[M_1\otimes \vct x_1\otimes \vct x_2])\nonumber\\
& &+ \frac{2\alpha_0^2}{(\alpha_0+1)(\alpha_0+2)}M_1\otimes M_1\otimes M_1,\\
M_y &=& \mathbb E[y\vct x_1\otimes \vct x_2] - \frac{\alpha_0}{\alpha_0+2}(\mathbb E[y] \mathbb E[\vct x_1\otimes\vct x_2]
+ \mathbb E[\vct x_1]\otimes\mathbb E[y\vct x_2] \nonumber\\
& & + \mathbb E[y\vct x_1]\otimes\mathbb E[\vct x_2]) + \frac{2\alpha_0^2}{(\alpha_0+1)(\alpha_0+2)}\mathbb E[y]M_1\otimes M_1.
\end{eqnarray}}\label{def_moment}
\end{deff}\vspace{-.25cm}
Note that the moments $M_1$, $M_2$ and $M_3$ were also defined in~\cite{a:twosvd,a:tensordecomp} for recovering the parameters of LDA models. For sLDA, we need to define a new moment $M_y$ in order to recover the linear regression model $\vct\eta$.
The moments are based on observable variables in the sense that they can be estimated from i.i.d. sampled documents.
For instance, $M_1$ can be estimated by computing the empirical distribution of all words,
and $M_2$ can be estimated using $M_1$ and word co-occurrence frequencies. Though the moments in the above forms look
complicated, we can apply elementary calculations based on the conditional independence structure of sLDA to significantly simplify them and more importantly to get them connected with the model parameters to be recovered, as summarized in Proposition~\ref{prop_moments}, whose proof is elementary and deferred to Appendix C for clarity.
%
\begin{prop}The moments can be expressed using the model parameters as:\\[-.6cm]
{\small \setlength\arraycolsep{1pt} \begin{eqnarray}
M_2 &=& \frac{1}{\alpha_0(\alpha_0+1)}\sum_{i=1}^k{\alpha_i\vct\mu_i\otimes\vct\mu_i}, \\
M_3 &=& \frac{2}{\alpha_0(\alpha_0+1)(\alpha_0+2)}\sum_{i=1}^k{\alpha_i\vct\mu_i\otimes\vct\mu_i\otimes\vct\mu_i},\\
M_y &=& \frac{2}{\alpha_0(\alpha_0+1)(\alpha_0+2)}\sum_{i=1}^k{\alpha_i\eta_i\vct\mu_i\otimes\vct\mu_i}.
\end{eqnarray}}\label{prop_moments}
\end{prop}

\begin{algorithm}[t]
\caption{spectral parameter recovery algorithm for sLDA. Input parameters: $\alpha_0, L, T$.}
\centering
\begin{algorithmic}[1]
	\State Compute empirical moments and obtain $\widehat M_2, \widehat M_3$ and $\widehat M_y$. 
	\State Find $\widehat W\in\mathbb R^{n\times k}$ such that $\widehat M_2(\widehat W, \widehat W) = I_k$.
	\State Find robust eigenvalues and eigenvectors $(\widehat\lambda_i,\widehat{\vct v}_i)$ of $\widehat M_3(\widehat W,\widehat W,\widehat W)$
	using the robust tensor power method \cite{a:tensordecomp} with parameters $L$ and $T$.
	\State Recover prior parameters: $\widehat\alpha_i \gets \frac{4\alpha_0(\alpha_0+1)}{(\alpha_0+2)^2\widehat\lambda_i^2}$.
	\State Recover topic distributions: \vspace{-.3cm}
$$\widehat{\vct\mu}_i\gets \frac{\alpha_0+2}{2}\widehat\lambda_i(\widehat W^+)^\top\widehat{\vct v}_i.$$\vspace{-.55cm}
	\State Recover the linear regression model: \vspace{-.2cm}
$$\widehat\eta_i \gets \frac{\alpha_0+2}{2}\widehat{\vct v}_i^\top\widehat M_y(\widehat W,\widehat W)\widehat{\vct v}_i.$$\vspace{-.55cm}
	\State \textbf{Output:} $\widehat{\vct\eta}$, $\widehat{\vct\alpha}$ and $\{\widehat{\vct\mu}_i\}_{i=1}^k$.
\end{algorithmic}
\label{alg1}
\end{algorithm}

\vspace{-.7cm}
\subsection{Simultaneous diagonalization}

Proposition~\ref{prop_moments} shows that the moments in Definition \ref{def_moment} are all the weighted sums of tensor products of $\{\vct\mu_i\}_{i=1}^k$ from the underlying sLDA model.
One idea to reconstruct $\{\vct\mu_i\}_{i=1}^k$ is to perform \emph{simultaneous diagonalization} on tensors of different orders.
The idea has been used in a number of recent developments of spectral methods for latent variable models~\cite{a:twosvd,a:tensordecomp,specregression}.
Specifically, we first whiten the second-order tensor $M_2$ by finding a matrix $W\in\mathbb R^{V\times k}$ such that
$W^\top M_2 W = I_k$.
This whitening procedure is possible whenever the topic distribuction vectors $\{\vct\mu_i\}_{i=1}^k$ are linearly independent
(and hence $M_2$ has rank $k$).
This is not always correct since in the ``overcomplete" case~\cite{a:overcomplete},
it is possible that the topic number $k$ is larger than vocabulary size $V$. However, the
linear independent assumption gives us a more compact representation for the topic model and
works well in practice. Hence we simply assume that the whitening procedure is possible.
The whitening procedure and the linear independence assumption also
imply that $\{W\vct\mu_i\}_{i=1}^k$ are orthogonal vectors (see Appendix A.2 for details),
and can be subsequently recovered by performing an orthogonal
tensor decomposition on the simultaneously whitened third-order tensor $M_3(W,W,W)$.
Finally, by multiplying the pseudo-inverse of the whitening matrix $W^+$ we obtain the topic distribution vectors $\{\vct\mu_i\}_{i=1}^k$.

It should be noted that Jennrich's algorithm \cite{jennrich1,jennrich2,ankur-review} could recover $\{\vct\mu_i\}_{i=1}^k$ directly from the 3-rd order tensor $M_3$ alone when $\{\vct\mu_i\}_{i=1}^k$ is linearly independent.
However, we still adopt the above simultaneous diagonalization framework because the intermediate vectors $\{W\vct\mu_i\}_{i=1}^k$
play a vital role in the recovery procedure of the linear regression model $\vct\eta$.

\vspace{-.1cm}
\subsection{The power update method}

Although the linear regression model $\vct\eta$ can be recovered in a similar manner by performing simultaneous diagonalization on $M_2$ and $M_y$, such a method has several disadvantages, thereby calling for novel solutions. First, after obtaining entry values $\{\eta_i\}_{i=1}^k$ we need to match them to the topic distributions $\{\vct\mu_i\}_{i=1}^k$ previously recovered.
This can be easily done when we have access to the true moments,
but becomes difficult when only estimates of observable tensors are available
because the estimated moments may not share the same singular vectors due to sampling noise.
A more serious problem is that when $\vct\eta$ has duplicate entries the orthogonal decomposition of $M_y$ is no longer unique.
Though a randomized strategy similar to the one used in~\cite{a:twosvd} might solve the problem,
it could substantially increase the sample complexity~\cite{a:tensordecomp} and render the algorithm impractical.

In \cite{spechmm}, the authors provide a method for the matching problem by reusing eigenvectors.
We here develop a power update method to resolve the above difficulties with a similar spirit. Specifically, after obtaining the whitened (orthonormal) vectors $\{\vct v_i\} \triangleq c_i\cdot W^{\top}\vct\mu_i$ \footnote{$c_i \triangleq \sqrt{\frac{\alpha_i}{\alpha_0(\alpha_0+1)}} $ is a scalar coefficient that depends on $\alpha_0$ and $\alpha_i$. See Appendix A.2 for details.} we recover the entry $\eta_i$ of the linear regression model directly by computing a power update $\vct v_i^\top M_y(W,W)\vct v_i$. In this way, the matching problem is automatically solved because we know what topic distribution vector $\vct\mu_i$ is used when recovering $\eta_i$. Furthermore, the singular values (corresponding to the entries of $\vct \eta$) do not need to be distinct because we are not using any unique SVD properties of $M_y(W,W)$. As a result, our proposed algorithm works for any linear model $\vct\eta$.

\vspace{-.1cm}
\subsection{Parameter recovery algorithm}

Alg.~\ref{alg1} outlines our parameter recovery algorithm for sLDA (Spectral-sLDA). First, empirical estimations of the observable moments in Definition \ref{def_moment} are computed from the given documents.
The simultaneous diagonalization method is then used
to reconstruct the topic distribution matrix $O$ and its prior parameter $\vct\alpha$. After obtaining $O=(\vct\mu_1,\cdots,\vct\mu_k)$,
we use the power update method introduced in the previous section to recover the linear regression model $\vct\eta$. We can also recover the noise level parameter $\sigma^2$ with the other
parameters in hand by estimating $\mathbb{E}[y]$ and $\mathbb{E}[y^2]$ since
$\mathbb{E}[y] = \mathbb{E}[\mathbb{E}[y|\bm{h}]] = \bm{\alpha}^{\top}\bm{\eta}$ and
$\mathbb{E}[y^2] = \mathbb{E}[\mathbb{E}[y^2|\bm{h}]] =
\bm{\eta}^{\top}\mathbb{E}[\bm{h}\otimes \bm{h}]\bm{\eta} + \sigma^2$, where the term $\mathbb{E}[\bm{h}\otimes \bm{h}]$ can be computed in an analytical form using the model parameters, as detailed in Appendix C.1.

Alg.~1 admits three hyper-parameters $\alpha_0$, $L$ and $T$. $\alpha_0$ is defined as the sum of all entries in the prior parameter $\vct\alpha$. Following the conventions in~\cite{a:twosvd,a:tensordecomp}, we assume that $\alpha_0$ is known a priori and use this value to perform parameter estimation. It should be noted that this is a mild assumption, as in practice usually a homogeneous vector $\vct\alpha$ is assumed and the entire vector is known~\cite{lsa}. The $L$ and $T$ parameters are used to control the number of iterations in the robust tensor power method. In general, the robust tensor power method runs in $O(k^3LT)$ time.
To ensure sufficient recovery accuracy, $L$ should be at least a linear function of $k$ and
$T$ should be set as $T = \Omega(\log(k) + \log\log(\lambda_{\max}/\varepsilon))$,
where $\lambda_{\max} = \frac{2}{\alpha_0+2}\sqrt{\frac{\alpha_0(\alpha_0+1)}{\alpha_{\min}}}$
and $\varepsilon$ is an error tolerance parameter.
Appendix~A.2 and~\cite{a:tensordecomp} provide a deeper analysis into the choice of $L$ and $T$ parameters.

\vspace{-.1cm}
\subsection{Sample Complexity Analysis}
\vspace{-.1cm}

We now analyze the sample complexity of Alg.~\ref{alg1} in order to achieve $\varepsilon$-error with a high probability. For clarity, we focus on presenting the main results, while deferring the proof details to Appendix~A, including the proofs of important lemmas that are needed for the main theorem.

\begin{thm}
Let $\sigma_1(\widetilde O)$ and $\sigma_k(\widetilde O)$ be the largest and the smallest singular values of the canonical topic distribution matrix $\widetilde O$.
Define $\lambda_{\min} \triangleq \frac{2}{\alpha_0+2}\sqrt{\frac{\alpha_0(\alpha_0+1)}{\alpha_{\max}}}$
and $\lambda_{\max} \triangleq \frac{2}{\alpha_0+2}\sqrt{\frac{\alpha_0(\alpha_0+1)}{\alpha_{\min}}}$‘∫?
with $\alpha_{\max}$ and $\alpha_{\min}$ the largest and the smallest entries of $\vct\alpha$.
Suppose $\widehat{\vct\mu}$, $\widehat{\vct\alpha}$ and $\widehat{\vct\eta}$ are the outputs of Algorithm 1,
and $L$ is at least a linear function of $k$.
Fix $\delta\in(0,1)$. For any small error-tolerance parameter $\varepsilon > 0$,
if Algorithm \ref{alg1} is run with parameter $T = \Omega(\log(k) + \log\log(\lambda_{\max}/\varepsilon))$
on $N$ i.i.d. sampled documents (each containing at least 3 words) with
$N\geq \max(n_1, n_2, n_3)$, where\vspace{-.2cm}
{\small 
	\begin{equation*}
	\begin{aligned}
	n_1 &= C_1\cdot \left(1+\sqrt{\log(6/\delta)}\right)^2\cdot \frac{\alpha_0^2(\alpha_0+1)^2}{\alpha_{\min}}, \nonumber \\
	n_2 &= C_2\cdot \frac{(1+\sqrt{\log(15/\delta)})^2}{\varepsilon^2\sigma_k(\widetilde O)^4} \\
	 &\cdot\max\left(\left( \|\vct\eta\| - \sigma \Phi^{-1}\left( \frac{\delta}{60}\right)\right)^2, \alpha_{\max}^2\sigma_1(\widetilde O)^2\right), \nonumber \\
	n_3 &= C_3\cdot \frac{(1+\sqrt{\log(9/\delta)})^2}{\sigma_k(\widetilde O)^{10}}\cdot\max\left(\frac{1}{\varepsilon^2},
			\frac{k^2}{\lambda_{\min}^2}\right), \nonumber
	\end{aligned}
	\end{equation*} }
and $C_1, C_2$ and $C_3$ are universal constants,
then with probability at least $1-\delta$, there exists a permutation $\pi:[k]\to [k]$ such that for every topic $i$, the following holds:\vspace{-.2cm}
\setlength\arraycolsep{-1pt} {\small \begin{eqnarray}
	&&|\alpha_i-\widehat{\alpha}_{\pi(i)}| \!\leq\! \frac{4\alpha_0(\alpha_0 \!+\! 1)(\lambda_{\max} \!+\! 5\varepsilon)}{(\alpha_0 \!+\! 2)^2\lambda_{\min}^2(\lambda_{\min} \!-\! 5\varepsilon)^2}\cdot 5\varepsilon,~if~\lambda_{\min} > 5\varepsilon \nonumber \\
	&&\|\vct\mu_i - \widehat{\vct\mu}_{\pi(i)}\| \!\leq\! \left(3\sigma_1(\widetilde O)\left(\frac{8\alpha_{\max}}{\lambda_{\min}} + \frac{5(\alpha_0+2)}{2}\right) + 1\right)\varepsilon \nonumber \\
    &&|\eta_i - \widehat\eta_{\pi(i)}| \!\leq\! \left(\frac{\|\vct\eta\|}{\lambda_{\min}} + (\alpha_0+2)\right) \varepsilon . \nonumber
\end{eqnarray}}\vspace{-.4cm}
\label{thm_main}

\end{thm}

In brevity, the proof is based on matrix perturbation lemmas (see Appendix~A.1) and analysis to the orthogonal tensor decomposition methods (including SVD and robust tensor power method) performed on inaccurate tensor estimations (see Appendix~A.2).
The sample complexity lower bound consists of three terms, from $n_1$ to $n_3$.
The $n_3$ term comes from the sample complexity bound for the robust tensor power method~\cite{a:tensordecomp};
the $(\|\vct\eta\|- \sigma \Phi^{-1}(\delta/60))^2$ term in $n_2$ characterizes the recovery accuracy for the linear regression model $\vct\eta$,
and the $\alpha_{\max}^2\sigma_1(\widetilde O)^2$ term arises when we try to recover the topic distribution vectors $\vct\mu$;
finally, the term $n_1$ is required so that some technical conditions are met.
The $n_1$ term does not depend on either $k$ or $\sigma_k(\widetilde O)$,
and could be largely neglected in practice.

%

\begin{rem}
An important implication of Theorem~\ref{thm_main} is that it provides a sufficient condition
for a supervised LDA model to be identifiable, as shown in Remark~\ref{rem_identifiability}.
To some extent, Remark~\ref{rem_identifiability} is the best identifiability result possible under our inference framework,
because it makes no restriction on the linear regression model $\vct\eta$, and the linear independence assumption is unavoidable
without making further assumptions on the topic distribution matrix $O$.
\end{rem}

\begin{rem} Given a sufficiently large number of i.i.d. sampled documents with at least 3 words per document, a supervised LDA model $\mathcal M=(\vct\alpha,\vct\mu,\vct\eta)$ is identifiable if $\alpha_0 = \sum_{i=1}^k{\alpha_i}$ is known and $\{\vct\mu_i\}_{i=1}^k$ are linearly independent.
\label{rem_identifiability}
\end{rem}



We now take a close look at the sample complexity bound in Theorem~\ref{thm_main}.
It is evident that $n_2$ can be neglected when the number of topics $k$ gets large,
because in practice the norm of the linear regression model $\vct\eta$ is usually assumed to be small in order to avoid overfitting.
Moreover, as mentioned before, the prior parameter $\vct\alpha$ is often assumed to be homogeneous with $\alpha_i = 1/k$ \cite{lsa}.
With these observations, the sample complexity bound in Theorem~\ref{thm_main}
can be greatly simplified. 
\begin{rem}
Assume $\|\vct\eta\|$ and $\sigma$ are small and $\vct\alpha = (1/k, \cdots, 1/k)$.
As the number of topics $k$ gets large, the sample complexity bound in Theorem~\ref{thm_main} can be simplified as\vspace{-.3cm}
\setlength\arraycolsep{1pt}\begin{equation}
N \geq \Omega\left( \frac{\log(1/\delta)}{\sigma_k(\widetilde O)^{10}}\cdot\max(\varepsilon^{-2}, k^3)\right).
\end{equation}\label{rem_simplebound1}
\end{rem}\vspace{-.6cm}

The sample complexity bound in Remark~\ref{rem_simplebound1} may look formidable
as it depends on $\sigma_k(\widetilde O)^{10}$.
However, such dependency is somewhat necessary because we are using third-order tensors to recover the underlying model parameters.



\section{Joint Parameter Recovery}

The above two-stage procedure has one possible disadvantages, that is, the recovery of the topic distribution matrix does not
use any supervision signal, and thus the recovered topics are often not good enough for prediction tasks, as shown in experiments.
The disadvantage motivates us to develop a joint spectral method with theoretical guarantees.
We now describe our single-phase algorithm.

\subsection{Moments of Observable Variables}
We first define some moments based on the observable variables including the information that we need to recover the model parameters.
Since we aim to recover the joint topic distribution matrix $O^*$, we combine the word vector $\bm{x}$ with the response variable $y$ to form a joint vector $ \bm{z} = [ \bm{x}', y]'$ and define the following moments:

\begin{deff} (Centerized Moments)
{\small \begin{equation}
\begin{aligned}
	& N_1 = \mathbb{E}[\bm{z}_1] \\
	& N_2 = \mathbb{E}[ \bm{z}_1  \otimes \bm{z}_2]  - \dfrac{\alpha_0}{\alpha_0 + 1 } N_1 \otimes N_1 - \sigma^2\bm{e} \otimes \bm{e}   \\
	& N_3 = \mathbb{E}[ \bm{z}_1 \otimes \bm{z}_2 \otimes \bm{z}_3] - \dfrac{\alpha_0}{\alpha_0 + 1}( \mathbb{E}[\bm{z}_1 \otimes \bm{z}_2 \otimes N_1]  \\
	& \ \ \ \ + \mathbb{E}[\bm{z}_1 \otimes N_1 \otimes \bm{z}_2] + \mathbb{E}[N_1 \otimes \bm{z}_1 \otimes \bm{z}_2] ) \\
	& \ \ \ \ + \dfrac{2\alpha_0^2}{(\alpha_0 + 1)(\alpha_0 +2)}N_1 \otimes N_1 \otimes N_1 \\
	& \ \ \ \ + 3\sigma^2N_{1_{V+1}}\bm{e}  \otimes \bm{e}  \otimes \bm{e}  
+  \dfrac{1}{\alpha_0 + 1} \sigma^2  (\bm{e} \otimes \bm{e} \otimes N_1 \\  
	& \ \ \ \ + \bm{e}  \otimes N_1 \otimes \bm{e}   + N_1 \otimes \bm{e}  \otimes \bm{e}  ),
\end{aligned}
\end{equation}}
where $\bm{e}$ is the $(V$+$1)$-dimensional vector with the last element equaling to $1$ and all others zero, $N_{1_{V+1}}$ is the $V+1$-th element of $N_1$. 
\end{deff}

The intuition for such definitions is derived from
an important observation that once the latent variable $\bm{h}$ is given, the mean value of $y$ is a weighted combination of the regression parameters and $\bm{h}$
(i.e., $\mathbb{E}[y|\bm{h}] = \sum_{i = 1}^{k}\eta_i h_i$), which has the same form as for $\bm{x}$ 
(i.e., $\mathbb{E}[\bm{x}|\bm{h}] = \sum_{i=1}^{k} {\bm \mu}_i h_i$). Therefore, it is natural to regard $y$ as an additional dimension of the word vector $\bm{x}$, which gives the new vector $\bm{z}$.
This combination leads to some other terms involving the high-order moments of $y$, which introduce the variance parameter $\sigma$ when we centerize the moments.
Although we can recover $\sigma$ in the two-stage method, 
recovering it jointly with the other parameters seems to be hard. Thus we treat $\sigma$ as a hyper-parameter.
One can determine it via a cross-validation procedure.

\begin{figure}[t]
\centering
\includegraphics[width = .5\textwidth]{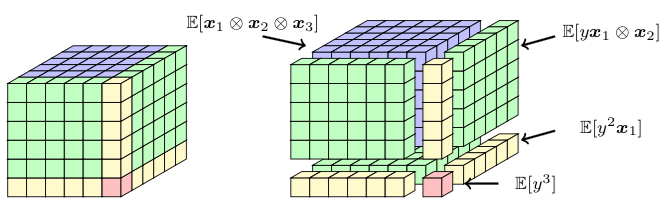}\vspace{-.2cm}
\caption{Interaction between response variable $y$ and word vector $\bm{x}$ in the 3rd-order tensor $M_3$.}
\label{fig10}\vspace{-.3cm}
\end{figure}

As illustrated in Fig.~\ref{fig10}, our 3rd-order moment can be viewed as a centerized version of the combination of $N_3' \triangleq \mathbb{E}[\bm{x}_1 \otimes \bm{x}_2 \otimes \bm{x}_3] $,
$N_y' \triangleq \mathbb{E}[y\bm{x}_2\otimes \bm{x}_2]$ and some high-order statistics of the response variables. 
Note that this combination has already aligned the regression parameters with the corresponding topics. Hence, we do not need an extra matching step. 

In practice, we cannot get the exact values of those moments. Instead, we estimate them from the i.i.d. sampled documents.
Note that we only need the moments up to the third order, which means any document consisting of at least three words can be used in this estimation.
Furthermore, although these moments seem to be complex, they can be expressed via the model parameters in a graceful manner, as
summarized in Proposition~\ref{p1} which can be proved by expanding the terms by definition, similar as in the proof of Proposition 1. 

\begin{prop} \label{p1}
	The moments in Definition 2 can be expressed by using the model parameters as follows:\vspace{-.2cm}
\begin{equation}
\begin{aligned}
	& N_2 = \dfrac{1}{\alpha_0(\alpha_0 + 1)} \sum\limits_{i =1}^{k} \alpha_i \bm{v}_i \otimes \bm{v}_i \\
	& N_3 = \dfrac{2}{\alpha_0(\alpha_0 +1)(\alpha_0 +2)} \sum\limits_{i=1}^{k} \alpha_i \bm{v}_i \otimes \bm{v}_i \otimes \bm{v}_i,
\end{aligned}\vspace{-.2cm}
\end{equation}
where $\bm{v}_i$ is the concatenation of the $i$-th word-topic distribution $\bm{\mu}_i$ and
	regression parameter $\eta_i$.
\end{prop}

\subsection{Robust Tensor Decomposition}

Proposition~\ref{p1} shows that the centerized tensors are weighted sums of the tensor products
of the parameters $\{\bm{v}_i\}_{i=1}^k$ to be recovered.
A similar procedure as in the two-stage method can be followed in order to develop our joint spectral method, which consists
of {\it whitening } and {\it robust tensor decomposition} steps.
First, we whiten the 2nd-order tensor $N_2$ by finding a matrix $W \in \mathbb{R}^ {(V +1) \times k}$ 
such that $W^{\top}N_2W = I_k$.
This whitening procedure is possible whenever the joint topic distribution vectors $\{\bm{v}_i\}^k_{i=1}$
are linearly independent, that is, the matrix has rank $k$. The whitening procedure and the linear independence
assumption also imply that $\{W^{\top}\bm{v}_i\}^k_{i=1}$
are orthogonal vectors  and
can be subsequently recovered by performing an orthogonal tensor decomposition on the simultaneously
whitened third-order tensor $N_3(W, W, W)$ as summarized in the following proposition.
\begin{prop}
	Define $\bm{\omega}_i = \sqrt{\dfrac{\alpha_i}{\alpha_0(\alpha_0 +1)}}W^{\top}\bm{v}_i$. Then:
	\begin{itemize}
		\item
			$\{\bm{\omega}\}_{i=1}^{k}$ is an orthonormal basis.
		\item
			$N_3(W,W,W)$  has pairs of robust eigenvalue and eigenvector $(\lambda_i, \bm{v}_i)$
			with $\lambda = \dfrac{2}{\alpha_0 + 2} \sqrt{\dfrac{\alpha_0(\alpha_0+1)}{\alpha_i}} $
	\end{itemize}
\end{prop}
Finally, by multiplying the pseudo-inverse of the whitening matrix $W^+$ we obtain the joint topic distribution vectors $\{\bm{v}_i\}^k_{i=1}$.

\begin{algorithm}[t]
	\caption{a joint spectral method to recover sLDA parameters. Input parameters: $\alpha_0$, $L$, $T$, $\sigma$ }\label{a1}
	\begin{algorithmic}[1]
		\State Compute empirical moments and obtain $\widehat{N}_2, \widehat{N_3}$. 
		\State Find $\widehat{W} \in \mathbb{R}^{V + 1\times k}$ such that $\widehat{N}_2(\widehat{W},\widehat{W}) = I_k $.
		\State Find robust eigenvalues and eigenvectors $(\widehat{\lambda}_i, \widehat{\bm{v}}_i)$
			of $\widehat{N}_3(\widehat{W},\widehat{W},\widehat{W})$ using the robust tensor
			power method with parameters $L$ and $T$.
		\State Recover prior parameters: $\widehat{\alpha}_i \gets \frac{4\alpha_0(\alpha_0 + 1)}{(\alpha_0+2)^2\widehat{\lambda}_i^2}.$
		\State Recover topic distribution:  $$\widehat{\bm{v}}_i \gets \frac{\alpha_0 + 2}{2}\widehat{\lambda}_i(W^+)^{\top}\widehat{\bm{\omega}}_i.$$	
		\State {\bf Output: } model parameters $\widehat{\alpha}_i$, $\widehat{\bm{v}}_i$ $i = 1,...,k$
	\end{algorithmic}
	\label{alg2}
\end{algorithm}

We outline our single-phase spectral method in Alg.~\ref{a1}. Here we assume that the noise variance $\sigma$ is given. Note that in the two-stage spectral method, it does
not need the parameter $\sigma$ because it does not use the information of the variance of prediction error.
Although there is a disadvantage that we need to tune it, the introduction of $\sigma$ sometimes increases
the flexibility of our methods on incorporating some prior knowledge (if exists).

We additionally need three hyper-parameters $\alpha_0, L $ and $T$, similar as in the two-stage method. The parameter $\alpha_0$ is defined as the summation of all the entries of the prior parameter $\bm{\alpha}$.
$L$ and $T$ are used to control the number of iterations in robust tensor decomposition. To ensure a sufficiently high recovery accuracy, $L$ should
be at least a linear function of $k$, and $T$ should be set as $T = \Omega(\log(k) + \log\log(\lambda_{max}/\epsilon))$,
where $\lambda_{max} = \frac{2}{\alpha_0 + 2}\sqrt{\frac{\alpha_0(\alpha_0 + 1)}{\alpha_{min}}}$ and $\epsilon$ is the error
rate. 

\vspace{-.1cm}
\subsection{Sample Complexity Analysis}
\vspace{-.1cm}

We now analyze the sample complexity in order to achieve $\epsilon$-error with a high probability.
For clarity, we defer proof details to Appendix B.

\begin{thm}
     Let $ \sigma_1(\widetilde{O}^*)$ and $\sigma_k(\widetilde{O}^*)$ be the largest and smallest singular
	 values of the joint canonical topic distribution matrix $\widetilde{O}^*$. Let
	 $\lambda_{max} \triangleq \frac{2}{\alpha_0 + 2}\sqrt{\frac{\alpha_0(\alpha_0 +1)}{\alpha_{min}}}$ where
	 $\alpha_{min}$ is the smallest element of $\bm{\alpha}$;
	 $\lambda_{min} \triangleq \frac{2}{\alpha_0 + 2}\sqrt{\frac{\alpha_0(\alpha_0 +1)}{\alpha_{max}}}$ where
	 $\alpha_{max}$ is the largest element of $\bm{\alpha}$.

     For any error-tolerance parameter $\epsilon > 0$, if Algorithm $2$ runs at least
	 $T \!=\! \Omega(\log(k) \!+\! \log\log(\lambda_{max}/\epsilon))$
	 iterations on $N$ i.i.d. sampled documents with $N \geq (n_1^\prime, n_2^\prime, n_3^\prime)$, where:\vspace{-.3cm}

 {\small  \begin{equation*}
   \begin{aligned}
         & n_1^\prime = K_1 \cdot  \dfrac{\alpha_0^2(\alpha_0+1)^2C^2(\delta/36N)\cdot(2 + \sqrt{2\log(18/\delta)})^2 }{\alpha_{min}^2}   \\
		 & n_2^\prime =  K_2 \cdot C^2(\delta/144N)(2 + \sqrt{2 \log(72/\sigma)})^2   \\
	     & n_3^\prime = K_3 \cdot \dfrac{C^2(\delta/36N)(2+\sqrt{2\log(18/\sigma)})^2}
		             {\sigma_k(\widetilde{O}^*)^{10}}\cdot\max(\dfrac{1}{\epsilon^2}, \dfrac{k^2}{\lambda_{min}^2}),\\
	     \end{aligned}
     \end{equation*}}\vspace{-.3cm}

	 \noindent  $C(x)$ is a polynomial of inverse CDF of normal distribution and the norm of regression parameters $\| {\bm \eta}\|$; $K_1,K_2,K_3$ are some universal constants.
	 Then with probability at least $1 - \delta$, there exist a permutation $\pi: [k] \to [k]$
     such that the following holds for every $i \in [k]$:\vspace{-.4cm}

     \begin{equation*}
	     \begin{aligned}
		     & |\alpha_i - \widehat{\alpha}_{\pi(i)} |
		             \leq \dfrac{4\alpha_0(\alpha_0 + 1) }{(\alpha_0 + 2)^2\lambda_{min}^2(\lambda_{min} - 5 \epsilon)^2} \cdot 5\epsilon \\
		             & \|\bm{v}_i - \widehat{\bm{v}}_{\pi(i)}  \|    \leq \left ( \sigma_1(\widetilde{O}^*)(\alpha_0  + 2)(\frac{7}{2} +
		     \frac{8}{\lambda_{min}})\right )\cdot \epsilon. 
	     \end{aligned}
     \end{equation*}\vspace{-.1cm}
	 \label{thm_main2}
\end{thm}

\vspace{-.5cm}
Similar to Theorem 1, the sample complexity bound consists of three terms.
The first and second terms do not
depend on the error rate $\epsilon$, which are required so that some technical conditions
are met. Thus they could be largely neglected in practice.
The third $n_3^\prime$ term comes from the sample complexity bound for the robust tensor power method~\cite{a:tensordecomp}.
\begin{rem}
	Note the RHS of the requirements of $N$ includes a
function of $N$~(i.e., $C(1/N)$). As mentioned above, $C(x)$ is polynomial of inverse CDF of normal
distribution with low degree. Since the inverse CDF grows very slowly~(i.e., $|\Phi^{-1}(1/N)| = o(\log(N))$).
We can omit it safely.
\end{rem}

\begin{rem}
Following the above remark and assume that $\|\bm{\eta}\|$ and $\sigma$ are small  and $\bm{\alpha}$ are homogeneous, the
sample complexity can be simplified as (a function of $k$):\vspace{-.3cm}
$$ 
N = O\left( \dfrac{\log(1/\sigma)}{\sigma_k(\widetilde{O}^*)^{10}} \cdot \max(\dfrac{1}{\epsilon^2}, k^3)\right).
$$\label{rem_simplebound2} \vspace{-.45cm}
\end{rem}\vspace{-.1cm}
The factor $1/\sigma_k(\widetilde{O}^*)^{10}$ is large, however, such a factor is necessary since we use the third order tensors. This factor roots in the tensor
decomposition methods and  one can expect to improve it if we have other better methods to decompose $\widehat{N}_3$.


\subsection{Sample Complexity Comparison}
\label{compp}
As mentioned in Remark \ref{rem_simplebound1} and Remark \ref{rem_simplebound2}, the joint spectral method shares the
same sample complexity as the two-stage algorithm in order to achieve $\epsilon$ accuracy, except two minor differences.

First, the sample complexity depends on the smallest singular value of (joint) topic distribution $\sigma_k(\widetilde{O})$. For the joint method, the joint topic distribution
matrix consists of the original topic distribution matrix and one extra row of the regression parameters. Thus from Weyl's inequality~\cite{a:weyl}, the smallest singular value
		of the joint topic distribution matrix is larger than that of the original topic distribution matrix, 
		and then the sample complexity of the joint method is a bit lower than that of the two-stage method, as empirically justified in experiments.

Second, different from the two-stage method, the errors of topic distribution $\bm{\mu}$ and regression parameters $\bm{\eta}$  are estimated together in the joint method 
 (i.e., $\bm{v}$), which can potentially give more accurate estimation of regression parameters considering that the number of
		regression parameters is much less than the topic distribution.

\section{Speeding up moment computation}

We now analyze the computational complexity and present some implementation details to make the algorithms more efficient.

\subsection{Two-Stage Method}

In Alg.~1, a straightforward computation of the third-order tensor $\widehat M_3$ requires
$O(NM^3)$ time and $O(V^3)$ storage, where $N$ is corpus size, $M$ is the number of words per document and $V$ is the vocabulary size.
Such time and space complexities are clearly prohibitive for real applications, where the vocabulary usually contains tens of thousands of terms.
However, we can employ a trick similar as in~\cite{speedup} to speed up the moment computation.
We first note that only the whitened tensor $\widehat M_3(\widehat W,\widehat W,\widehat W)$ is needed in our algorithm, which only takes $O(k^3)$ storage.
Another observation is that the most difficult term in $\widehat M_3$ can be written as
$\sum_{i=1}^r{c_i\vct u_{i,1}\otimes \vct u_{i,2}\otimes\vct u_{i,3}}$,
where $r$ is proportional to $N$ and $\vct u_{i,\cdot}$ contains at most $M$ non-zero entries.
This allows us to compute $\widehat M_3(\widehat W,\widehat W,\widehat W)$ in $O(NMk)$ time by computing
$\sum_{i=1}^r{c_i\vct (W^\top\vct u_{i,1})\otimes (W^\top\vct u_{i,2})\otimes (W^\top\vct u_{i,3})}$.
Appendix~C.2 provides more details about this speed-up trick.
The overall time complexity is $O(NM(M+k^2)+V^2+k^3LT)$
and the space complexity is $O(V^2+k^3)$.

\subsection{Joint Method}

For the single-phase algorithm, a straightforward computation of the third-order tensor $\widehat N_3$ has the same complexity of $O(NM^3)$ as in the two-stage method.
And a much higher time complexity is needed for computing $\widehat{N}_3(\widehat{W},\widehat{W},\widehat{W})$
, which is prohibitive.
Similar as in the two-stage method, since we only need $\widehat{N}_3(\widehat{W},\widehat{W},\widehat{W})$ in Alg.~\ref{a1}, we turn to compute this term directly.
We can then use the trick mentioned above to do this.
The key idea is to decompose the third-order tensor into different parts based on the occurrence of words and compute them respectively.
The same time comlexity and space complexity is needed for the single-phase method.

Sometimes $N_2$ and $N_3$ are not ``balanced" (i.e., the value of some elements are much larger than the others). 
This situation happens when either
the vocabulary size is too large or the range of $\bm{\eta}$ is too large. One can image that
if we have a vocabulary consisting of one million words while $\min{\eta}_i = 1$, then the energy
of the matrix $N_2$ concentrates on $(N_2)_{V+1,V+1}$.
As a consequence,  the SVD performs badly when the matrix is ill-conditioned.
A practical solution
to this problem is that we scale the word vector $\bm{x}$ by a constant, that is, for the $i$-th word in the dictionary, we set $x_i = C , x_j = 0, \forall i \not = j$, where $C$ is a constant.
The main effect is that we can make the matrix more stable after this manipulation. Note that when we fix $C$,
this makes no effect on the recovery accuracy. Such a trick is primarily for computational stability. In our experiments, $C$ is set to be $100$.

\subsection{Dealing with large vocabulary size $V$}
One key step in the whitening procedure of both methods is to perform SVD on the second order moment $M_2 \in \mathbb{R}^{V \times V}$ (or $N_2 \in \mathbb{R}^{(V+1) \times (V+1)}$).
A straightforward implementation of SVD has complexity $O(k V^2)$,\footnote{It is not $O(V^3)$ as we only need top-$k$ truncated SVD.} which is unbearable when the vocabulary size $V$ is large.
We follow the method in \cite{a:nystrom} and perform random projection to reduce dimensionality. More precisely, let $S \in \mathbb{R}^{\widetilde{k}}$ where $\widetilde{k} < V$ be
a random matrix and then define $C = M_2 S$ and $\Omega = S^{\top}M_2S$. Then a low rank approximation of $M_2$ is given by $ \widetilde{M_2} = C\Omega^{+}C^{\top} $.
Now we can obtain the whitening matrix without directly performing an SVD on $M_2$ by appoximating $C^{-1}$ and $\Omega$ separately.
The overall algorithm is provided in Alg.~\ref{alg3}. In practice, we set $\widetilde{k} = 10 k$ to get a sufficiently accurate approximation.

\begin{algorithm}[t]
	\caption{Randomized whitening procedure. Input parameters: second order moment $M_2$ (or $N_2$).}
\centering
\begin{algorithmic}[1]
	\State Generate a random projection matrix $S \in \mathbb{R}^{V\times \widetilde{k}}$.
	\State Compute the matrices $C$ and $\Omega$:\vspace{-.3cm}
$$C = M_2S \in \mathbb{R}^{V \times \widetilde{k}},~\textrm{and}~ \Omega = S^{\top}M_2S \in \mathbb{R}^{ \widetilde{k}\times \widetilde{k}}.$$\vspace{-.7cm}
	\State Do SVD for both $C$ and $\Omega$:\vspace{-.3cm}
$$C = U_C \Sigma_C D_C^\top,~\textrm{and}~ \Omega = U_\Omega \Sigma_\Omega D^\top_\Omega$$\vspace{-.7cm}
	\State Take the rank-$k$ approximation:  $U_C \leftarrow U_C(\colon ,1\colon k)$ \vspace{-.3cm}
$$\Sigma_C \leftarrow \Sigma_C(1\colon k,1\colon k), D_C \leftarrow D_C(\colon,1\colon k)$$ \vspace{-.6cm}
$$ D_{\Omega} \leftarrow D_{\Omega}(1\colon k,1\colon k),
			\Sigma_{\Omega} \leftarrow \Sigma_{\Omega}(1\colon k,1 \colon k)$$\vspace{-.7cm}
	\State Whiten the approximated matrix: \vspace{-.3cm}
$$W = U_C\Sigma_C^{-1}D_C^{\top}D_{\Omega}\Sigma_{\Omega}^{1/2}. $$\vspace{-.7cm}
	\State \textbf{Output:} Whitening matrix $W$.
\end{algorithmic}
\label{alg3}
\end{algorithm}

\section{Experiments}\label{sec:exp}
We now present experimental results on both synthetic and two real-world datasets.
For our spectral methods,
the hyper-parameters $L$ and $T$ are set to be $100$, which is sufficiently large for our experiment settings.
Since spectral methods can only recover the underlying parameters, we first run them to
recover those parameters in training and then use Gibbs sampling to infer the topic mixing vectors $\bm{h}$ and topic assignments for each word $t_i$ for testing.

Our main competitor is sLDA with a Gibbs sampler (Gibbs-sLDA), which is asymptotically accurate and often outperforms variational methods.
We implement an uncollapsed Gibbs sampler, which alternately draws samples from the local
conditionals of $\bm{\eta}$, $\vct z$, $\hv$, or $\bm{\mu}$,
when the rest variables are given. We monitor the behavior of the
Gibbs sampler by observing the relative
change of the training data log-likelihood, and terminate when the average change
is less than a given threshold (e.g., $1e^{-3}$) in the last $10$ iterations.
The hyper-parameters of the Gibbs sampler are set to be the same as our methods, including topic numbers and $\bm{\alpha}$.
We evaluate a hybrid method that uses the parameters $\bm{\mu},\bm{\eta}, \bm{\alpha}$ recovered by our joint spectral method
as initialization for a Gibbs sampler. 
This strategy is similar to that in~\cite{a:meetem}, where
the estimation of a spectral method is used to initialize an EM method for further refining. In our hybrid method, the Gibbs sampler plays the similar role of refining.
We also compare with MedLDA~\cite{medlda}, a state-of-the-art topic model for classification and regression, on real datasets. We use the Gibbs sampler with data augmentation~\cite{gibbsmedlda}, which is more accurate than the original variational methods, and adopts the same stopping condition as above.

On the synthetic data, we first use $L_1$-norm
to measure the difference between the reconstructed parameters
and the underlying true parameters.
Then we compare the prediction accuracy and per-word likelihood on both synthetic and real-world datasets.
The quality of the prediction on the synthetic dataset is measured by mean squared error (MSE) while
the quality on the real-word dataset is assessed by predictive $R^2$ ($pR^2$), a normalized version of MSE, which is defined as 
$pR^2 = 1 - \frac{ \sum_i(y_i - \widehat{y}_i)^2 }{ \sum_i(y_i - \bar{y})^2 },$
where $\bar{y}$ is the mean of
testing data and $\widehat{y}_i$ is the estimation of $y_i$.  The per-word log-likelihood
is defined as $\log p(\omega|\hm{h},O) = \log \sum_{j=1}^{k}p(\omega|t =j, O)p(t = j|\bm{h})$. 

\begin{figure*}[t]\vspace{-.1cm}
\centering
\includegraphics[width=4.9cm]{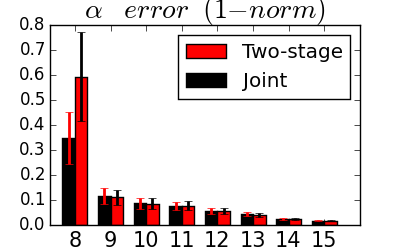}
\includegraphics[width=4.9cm]{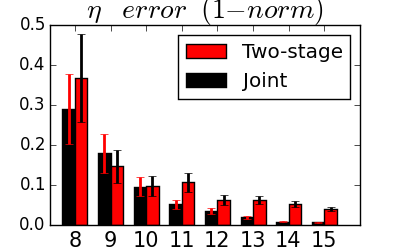}
\includegraphics[width=4.9cm]{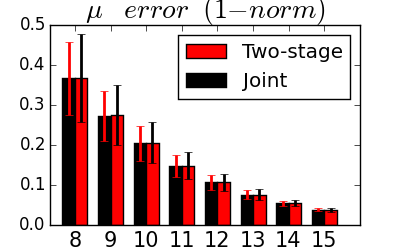}\vspace{-.2cm}
\caption{Reconstruction errors of two spectral methods when each document contains $250$ words.
	$X$ axis denotes the training size $n$ in log domain with base $2$ (i.e., $n = 2^k, k \in \{8,...,15\}$). Error bars denote the standard deviations measured on 3 independent trials under each setting.}
\label{fig_convergence1}\vspace{-.1cm}
\end{figure*}

\begin{figure*}[t]
\centering
\includegraphics[width=4.9cm]{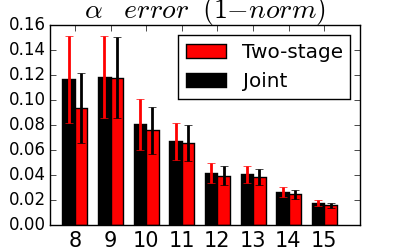}
\includegraphics[width=4.9cm]{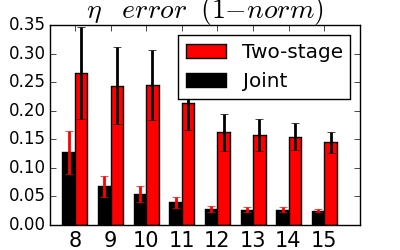}
\includegraphics[width=4.9cm]{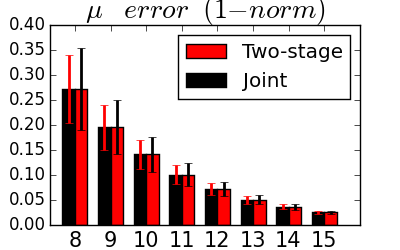}\vspace{-.3cm}
\caption{Reconstruction errors of two spectral methods when each document contains $500$ words.
$X$ axis denotes the training size $n$ in log domain with base $2$ (i.e., $n = 2^k, k \in \{8,...,15\}$). Error bars denote the standard deviations measured on 3 independent trials under each setting.}
\label{fig_convergence2}\vspace{-.45cm}
\end{figure*}

\vspace{-.15cm}
\subsection{Synthetic Dataset}\label{sec:synthetic}
We generate our synthetic dataset following the generative process of sLDA, with a vocabulary of
size $V = 500$ and topic number $k = 20$. 
We generate the topic distribution matrix $O$ by first sampling each entry from a uniform distribution and then normalizing
every column of it. The linear regression model $\bm{\eta}$ is sampled from a standard Gaussian distribution.
The prior parameter $\bm{\alpha}$ is assumed to be homogeneous, i.e., $\bm{\alpha} = (1/k,... , 1/k)$.
Documents and response variables are then generated from the sLDA model specified in Section~\ref{sec:sLDA-model}.
We consider two cases where the length of each document is set to be $250$ and $500$ repectively.
The hyper-parameters are set to be the same as the ones that used to
generate the dataset 
\footnote{The  methods are insensitive to
the hyper-parameters in a wide range. e.g., we still get high accuracy even we set the hyper-parameter $\alpha_0$ to be twice as large as the true value.}.

\vspace{-.15cm}
\subsubsection{Convergence of estimated model parameters}

Fig.~\ref{fig_convergence1} and Fig.~\ref{fig_convergence2} show the $L_1$-norm reconstruction errors of $\bm{\alpha}$, $\bm{\eta}$ and $\bm{\mu}$ when each document contains different number of words. 
Note that due to the unidentifiability of topic models, we only get a permutated estimation of the underlying parameters. Thus we run a bipartite
graph matching to find a permutation that minimizes the reconstruction error. We can find that as the sample size increases, the reconstruction errors for all parameters decrease consistently to zero in both methods, which verifies the correctness of our theory.
Taking a closer look at the figures, we can see that the empirical convergence rates for $\bm{\alpha}$ and $\bm{\mu}$ are
almost the same for the two spectral methods. However, the convergence rate for regression parameters $\bm{\eta}$ in the joint
method is much higher than the one in the two-stage method, as mentioned in the comparison
of the sample complexity in Section~\ref{compp} , due to the fact that the joint method can bound the
estimation error of $\bm{\eta}$ and $\bm{\mu}$ together.

Furthermore, though Theorem~\ref{thm_main} and Theorem~\ref{thm_main2} do not involve the number of words per document,
the simulation results demonstrate a significant improvement when more words are observed in each document,
which is a nice complement for the theoretical analysis.

\vspace{-.15cm}
\subsubsection{Prediction accuracy and per-word likelihood}

Fig.~\ref{fig_prediction1} shows that both spectral methods consistently outperform Gibbs-sLDA.
Our methods also enjoy the advantage of being less variable,
as indicated by the curve and error bars.
Moreover, when the number of training documents is sufficiently large, the performance of the reconstructed model is very close to the true model\footnote{Due to the randomness in the data generating process, the true model has a non-zero prediction error.}, which implies that our spectral methods
can correctly identify an sLDA model from its observations, therefore supporting our theory.

The performances of the two-stage spectral method and the joint one are comparable this time, which is largely because of the
fact the when giving enough training data, the recovered model is accurate enough. The Gibbs method is easily caught in a local
minimum so we can find as the sample size increases, the prediction errors do not decrease monotonously.

\begin{figure}[t]\vspace{-.1cm}
\centering
\includegraphics[width=4cm]{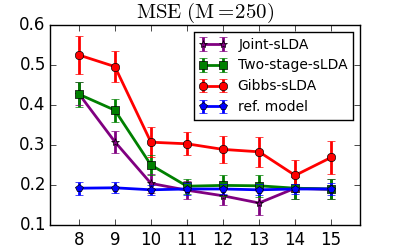}
\includegraphics[width=4cm]{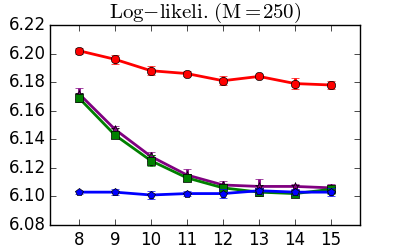}
\includegraphics[width=4cm]{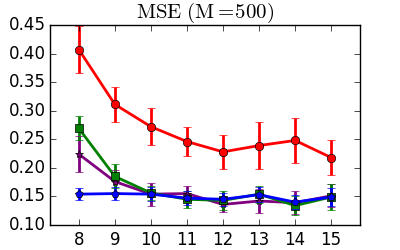}
\includegraphics[width=4cm]{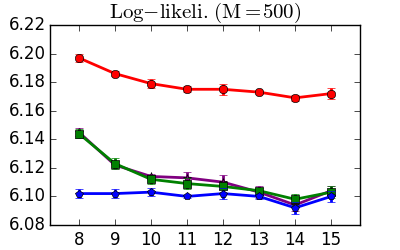}\vspace{-.2cm}
\caption{
Mean square errors and negative per-word log-likelihood of Alg.~\ref{alg1} and Gibbs sLDA.
Each document contains $M=500$ words. The $X$ axis denotes the training size ($\times 10^3$).
The ''ref. model" denotes the one with the underlying true parameters.
}
\label{fig_prediction1}\vspace{-.2cm}
\end{figure}

\begin{figure}[ht!]
\centering
\includegraphics[width=7cm]{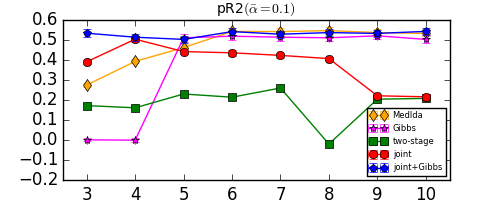}
\includegraphics[width=7cm]{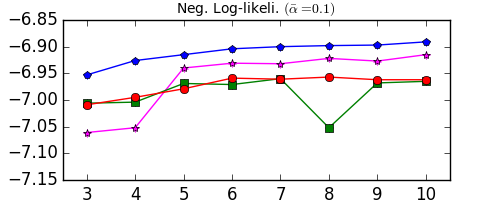}\vspace{-.2cm}
\caption{$pR^2$ scores and negative per-word log-likelihood on the Hotel Review dataset. The $X$ axis indicates the number of topics. Error bars
indicate the standard deviation of $5$-fold cross-validation. }
\label{hotel}\vspace{-.4cm}
\end{figure}
\subsection{Hotel Reviews Dataset}
For real-world datasets, we first test on a relatively small Hotel Review dataset, which consists of $15,000$ documents for training and $3,000$ documents for
testing that are randomly sampled from TripAdvisor website. Each document is associated with
a rating score from $1$ to $5$ and our task is to predict it.
We pre-process the dataset by shifting the review scores so that they have zero mean and unit variance as in~\cite{gibbsmedlda}.

Fig.~\ref{hotel} shows the prediction accuracy and per-word likelihood when the vocabulary size
is $5,000$ and the mean level of $\bm{\alpha}$ is $\hat{\alpha} = 0.1$. As MedLDA adopts a quite different objective from sLDA, we only compare on the prediction accuracy. Comparing with
traditional Gibbs-sLDA and MedLDA, the two-stage spectral method is much worse, while
the joint spectral method is comparable at its optimal value. This result is not surprising
since the convergence rate of regression parameters for the joint method is faster than
that of the two-stage one.
The hybrid method (i.e., Gibbs sampling initialized with the joint spectral method) performs as
well as the state-of-the-art MedLDA. These results show that spectral methods are
good ways to avoid stuck in relatively bad local optimal solution.

\vspace{-.1cm}
\subsection{Amazon Movie Reviews Dataset}

Finally, we report the results on a large-scale real dataset, which is built on Amazon movie reviews~\cite{a:amazon}, to demonstrate
the effectiveness of our spectral methods on improving the prediction accuracy as well as finding discriminative topics.
The dataset consists of $7,911,684$ movie reviews written by $889,176$ users from Aug $1997$ to Oct $2012$. Each review is accompanied
with a rating score from $1$ to $5$ indicating how a user likes a particular movie. The median number of words per
review is $101$. We consider two cases where a vocabulary with $V = 5,000$ terms or $V = 10,000$ is built by selecting high frequency words and deleting
the most common words and some names of characters in movies.
When the vocabulary size $V$ is small (i.e., $5,000$), we run exact SVD for the
whitening step; when $V$ is large (i.e., $10,000$), we run the randomized SVD to approximate the result.
As before, we also pre-process the dataset by shifting the review scores so that they have zero mean and unit variance.

\begin{figure*}[ht!]\vspace{-.1cm}
\centering
\includegraphics[width=7.3cm]{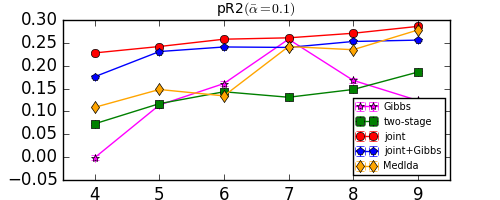}
\includegraphics[width=7.3cm]{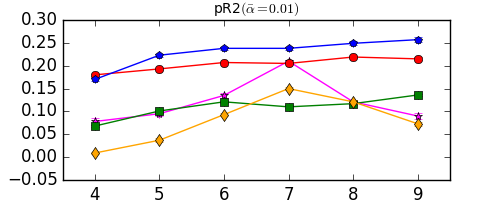}
\includegraphics[width=7.3cm]{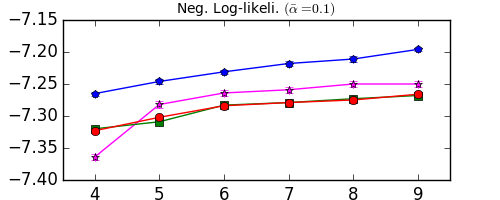}
\includegraphics[width=7.3cm]{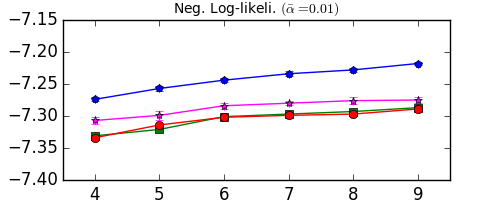}\vspace{-.2cm}
\caption{$pR^2$ scores and negative per-word log-likelihood on Amazon dataset. The $X$ axis indicates the number of topics. Error bars
indicate the standard deviation of $5$-fold cross-validation. Vocabulary size $V = 5,000$} 
\label{fig1}
\end{figure*}

\begin{figure*}[ht!]\vspace{-.1cm}
\centering
\includegraphics[width=7.3cm]{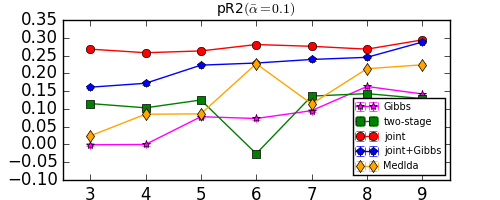}
\includegraphics[width=7.3cm]{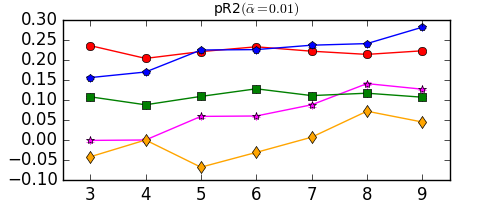}
\includegraphics[width=7.3cm]{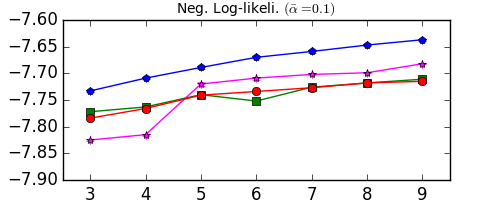}
\includegraphics[width=7.3cm]{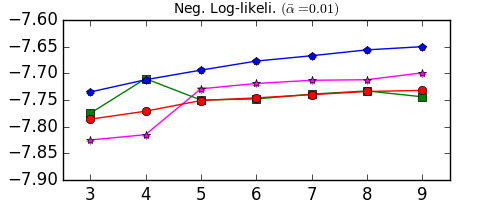}\vspace{-.2cm}
\caption{$pR^2$ scores and negative per-word log-likelihood on Amazon dataset. The $X$ axis indicates the number of topics. Error bars
indicate the standard deviation of $5$-fold cross-validation. Vocabulary size $V = 10,000$} 
\label{fig2}\vspace{-.4cm}
\end{figure*}

\vspace{-.15cm}
\subsubsection{Prediction Performance}

Fig.~\ref{fig1} shows the prediction accuracy and per-word log-likelihood when $\bar{\alpha}$ takes different values and the vocabulary size $V = 5,000$, where
$\bar{\alpha} = \alpha_0/k$ denotes the mean level for $\bm{\alpha}$.
We can see that comparing to the classical Gibbs sampling method, our spectral method is
a bit more sensitive to the hyper-parameter $\bm{\alpha}$.
But in both cases, our joint method alone outperforms the Gibbs sampler and the two-stage spectral method.
MedLDA is also sensitive to the hyper-parameter $\bm{\alpha}$.
When $\bm{\alpha}$ is set properly, MedLDA achieves the best result comparing with the other
methods, however, the gap between our joint method and MedLDA is small.
This result is significant for spectral methods, whose practical performance was often much inferior to likelihood-based estimators.
We also note that if $\bar{\alpha}$ is not set properly (e.g., $\bar{\alpha} = 0.01$), a hybrid method that initializes a Gibbs sampler by the results of our spectral methods
can lead to high accuracy, outperforming the Gibbs sampler and MedLDA
with a random initialization.
We use the results of the joint method for initialization because this
gives better performance compared with the two-stage method.

Fig.~\ref{fig2}
shows the results when the vocabulary size $V = 10,000$.
This time the joint spectral method gets the best result while the two-stage method
is comparable with Gibbs sampling but worse than MedLDA. 
The hybrid method is comparable with the joint method, demonstrating that this strategy works
well in practice again.
An interesting phenomenon is that the spectral method
gets good results when the topic number is only $3$ or $4$, which means the spectral method can fit the data using fewer topics.
Although there is a rising trend on prediction accuracy
for the hybrid method, we cannot verify this because we cannot get the results of spectral methods when $k$ is large.
The reason is that when $k > 9$, the spectral method fails in the robust
tensor decomposition step, as we get some negative eigenvalues. This phenomenon can be explained by the nature of our methods --- one crucial step in Alg.~\ref{alg1} and Alg.~\ref{a1} 
is to whiten $\widehat{M}_2$ which can be done when the underlying topic matrix $O$ ( or joint topic matrix $\widehat{O}^*$) is of full rank. For the Amazon
review dataset, it is impossible to whiten it with more than $9$ topics. This fact can be used for model selection to avoid using too many extra topics.
There is also a rising trend in the Gibbs sampling when $V = 10,000$ as we measure the $pR^2$ indicator,
it reaches peak when topic size $k = 40$ which is about $0.22$ no matter $\alpha_0$ is $0.1$ or $0.01$.
The results may indicate that with a good initialization, the Gibbs sampling method could get much better performance.

Finally, note that Gibbs sampling and the hybrid Gibbs sampling methods get better log-likelihood values. This result is not
surprising because Gibbs sampling is based on MLE while spectral methods do not. Fig.~\ref{fig1} shows that the hybrid Gibbs
sampling achieves the best per-word likelihood. Thus if one's main goal is to maximize likelihood, a hybrid technique is
desirable.

\vspace{-.15cm}
\subsubsection{Parameter Recovery}

We now take a closer investigation of the recovered parameters for our spectral methods.
Table $1$ shows the estimated regression parameters of both methods, with $k=8$ and $V=5,000$.
We can see that the two methods have different ranges of the possible predictions ---
due to the normalization of ${\bm h}$, the range of the predictions by a model with estimate $\widehat{ {\bm \eta} }$ is $[\min( \widehat{ {\bm \eta}}), \max(\widehat{{\bm \eta}})]$.
Therefore, compared with the range provided by the two-stage method (i.e., $[-0.75,0.83]$), the
joint method gives a larger one (i.e. $[-2.00,1.12]$) which better matches the range of the true labels (i.e., $[-3, 1]$) and therefore
leads to more accurate predictions as shown in Fig.~\ref{fig1}. 

We also examine the estimated topics by both methods.
For the topics with the large value of $\eta$, positive words (e.g., ``great")
dominate the topics in both spectral methods because the frequencies for them are much higher than negative ones (e.g., ``bad"). Thus we mainly focus on the ``negative" topics where the
difference can be found more expressly.
Table 2 shows the topics correspond to the smallest value of ${\bm \eta}$ by each method. 
To save space, for the topic in each method we show the non-neutral words from the top $200$ with highest probabilities.
For each word, we show its probability as well as the rank (i.e., the number in bracket)
in the topic distribution vector.
\begin{wraptable}{r}{.23\textwidth} \vspace{-.5cm} 
\caption{Estimated $\etav$ by the two spectral methods (sorted for ease of comparison).}\label{table:param} \vspace{-.2cm} 
\centering
\begin{tabular}{|c|c|}
    \hline\hline
     Two-stage & Joint \\
    \hline
    -0.754 & -1.998 \\
    -0.385 & -0.762 \\
    -0.178 & -0.212  \\
    -0.022 & -0.098 \\
    0.321  & 0.437  \\
    0.522  &  0.946 \\
    0.712  & 1.143 \\
    0.833  & 1.122 \\
    \hline
\end{tabular}\vspace{-.3cm}
\end{wraptable}
We can see that the negative words (e.g., ``bad", ``boring") have a higher rank (on average) in the topic by the joint spectral method than in the topic by the two-stage method,
while the positive words (e.g., ``good", ``great") have a lower rank (on average) in the topic by the joint method than in the topic by the two-stage method.
This result suggests that this topic in the joint method is more strongly associated with the negative reviews, therefore yielding a
better fit of the negative review scores when combined with the estimated ${\bm \eta}$.
Therefore, considering the supervision information can lead to improved topics.
Finally, we also observe that in both topics some positive words (e.g., ``good") have a rather high rank.
This is because the occurrences of such positive words are much frequent than the negative ones.

\begin{table}[t]\vspace{-.15cm}
\caption{Probabilities and ranks (in brackets) of some non-neutral words. N/A means that the word does not appear in the top $200$ ones with highest probabilities.}\vspace{-.2cm}
\centering
\begin{tabular}{|c|l|l|}
\hline\hline
   words & Two-stage spec-slda & Joint spec-slda \\
\hline
bad  & $ 0.006905 ~(14) $ & $ 0.009864 ~(8) $ \\
boring & $ 0.002163 ~(101) $ & $ 0.002433 ~(63) $ \\
stupid & $ 0.001513 ~(114) $ & $ 0.001841 ~(93) $ \\
horrible & $ 0.001255 ~(121) $ & $ 0.001868 ~(89) $ \\
terrible & $0.001184 ~(136) $ & $0.001868 ~(88) $ \\
waste & $0.001157 ~(140) $ & $ 0.001896 ~(84) $ \\
disappointed & $ 0.000926 ~(171)$ & $0.001282 ~(127) $ \\
\hline
good  & $0.012549 ~(7) $ & $0.012033 ~(5) $ \\
great & $0.012549 ~(11) $ & $0.003568 ~(37) $ \\
love  & $0.007513 ~(12)  $ & $0.003137 ~(46)$ \\
funny & $ 0.004334 ~(26) $ & $ 0.003346 ~(39) $ \\
enjoy & $ 0.002163 ~(77) $ & $ 0.001317 ~(123) $ \\
awesome & $0.001208 ~(132) $ & N/A  \\
amazing & $0.001199 ~(133) $ & N/A  \\
\hline
\end{tabular}\vspace{-.4cm}
\end{table}

\vspace{-.15cm}
\subsection{Time efficiency}
Finally, we compare the time efficiency with Gibbs sampling.
All algorithms are implemented in C++.

Our methods are very time efficient
because they avoid the time-consuming iterative steps in traditional variational inference and Gibbs sampling methods.
Furthermore, the empirical moment computation, which is the most time-consuming part in Alg.~\ref{alg1} and Alg.~\ref{alg2} when dealing with large-scale datasets,
consists of only elementary operations and can be easily optimized.
Table~\ref{table:time} shows the running time on the synthetic dataset with various sizes in the setting where the topic number is $k=10$, vocabulary size is $V=500$ and document length is $100$. 
We can see that both spectral methods are much faster than Gibbs sampling, especially when the data size is large.

Another advantage of our spectral methods is that we can easily parallelize the
computation of the low-order moments over multiple compute nodes, followed by a single step of synchronizing the local moments.
Therefore, the communication cost will be very low, as compared to the distributed algorithms for topic models~\cite{a:yahoolda} 
which often involve intensive communications in order to synchronize the messages for (approximately) accurate inference.

\begin{table}[t]\vspace{-.1cm}
\caption{Running time (seconds) of our spectral learning methods and Gibbs sampling.}
\label{table:time}\vspace{-.2cm}
\centering
\begin{tabular}{lllllll}
\hline
	$n(\times 5 \times 10^3)$ & 1 & 2 & 4 & 8 & 16 & 32 \\
\hline
Gibbs sampling&  47 & 92 & 167 & 340 & 671 & 1313\\
Joint spec-slda&  11 & 15 & 17 & 28 & 45 & 90\\
Two-stage spec-slda& 10& 13& 15& 22 & 39 & 81\\
\hline
\end{tabular}\vspace{-.3cm}
\end{table}

As a small $k$ is sufficient for the Amazon review dataset, we report the results with different $k$ values on a synthetic dataset where the vocabulary size $V=500$, the document length $m=100$ and the document size $n=1,000,000$. As shown in Fig.~\ref{fig:distribute}, the distributed implementation of our spectral methods (both two-stage and joint) has almost ideal (i.e., linear) speedup
with respect to the number of threads for moments computing.   
The computational complexity of the tensor decomposition step is $O(k^{5+\delta})$ for a third-order tensor $T\in \mathbb{R}^{k \times k \times k}$, where $\delta$ is small~\cite{a:tensordecomp}.  
When the topic number $k$ is large (e.g., as may be needed in applications with much larger datasets), one can follow the
recent developed stochastic tensor gradient descent (STGD) method to compute the
eigenvalues and eigenvectors~\cite{a:sgdtd}, which 
can significantly reduce the running time in the tensor decomposition stage.

\begin{figure}[t]
\centering
\includegraphics[width=0.36\textwidth,height=0.25\textwidth]{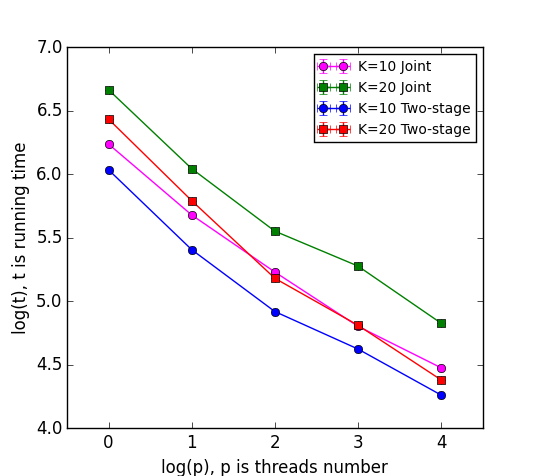}\vspace{-.3cm}
\caption{Running time of our method w.r.t the number of threads. Both $x$ and $y$ axes are plotted in log scale with base $e$.}
\label{fig:distribute}\vspace{-.5cm}
\end{figure}


\vspace{-.15cm}
\section{Conclusions and Discussions}

We propose two novel spectral decomposition methods to recover the parameters of supervised LDA models from labeled documents.
The proposed methods enjoy a provable guarantee of model reconstruction accuracy and are highly efficient
and effective.
Experimental results on real datasets demonstrate that
the proposed methods, especially the joint one, are superior to existing methods.
This result is significant for spectral methods, which were often inferior to MLE-based methods in practice.
For further work, it is interesting to recover parameters when the regression model is non-linear.


%

\section*{Acknowledgements}

This work is supported by the National 973 Basic Research Program of China (Nos. 2013CB329403, 2012CB316301), National NSF of China (Nos. 61322308, 61332007), and Tsinghua Initiative Scientific Research Program (No. 20141080934).

\bibliographystyle{plain}
\bibliography{spectral-slda}

\appendices

\onecolumn

\include{supplementary1}

\include{supplementary2}

\include{supplementary3}

\end{document}

%% file: supplementary1.tex
\section{Proof to Theorem 1}
In this section, we prove the sample complexity bound given in Theorem 1. The proof consists of three main parts.
In Appendix A.1, we prove perturbation lemmas that bound the estimation error of the whitened tensors
$M_2(W,W), M_y(W,W)$ and $M_3(W,W,W)$ in terms of the estimation error of the tensors themselves.
In Appendix A.2, we cite results on the accuracy of SVD and robust tensor power method when performed on estimated tensors,
and prove the effectiveness of the power update method used in recovering the linear regression model $\vct\eta$.
Finally, we give tail bounds for the estimation error of $M_2,M_y$ and $M_3$ in Appendix A.3
and complete the proof in Appendix A.4.
We also make some remarks on the indirect quantities (e.g. $\sigma_k(\widetilde O)$) used in Theorem 1
and simplified bounds for some special cases in Appendix A.4.

All norms in the following analysis, if not explicitly specified, are 2 norms in the vector and matrix cases and the operator norm in the high-order tensor case.

\subsection{Perturbation lemmas}

We first define the canonical topic distribution vectors $\widetilde{\vct\mu}$ and estimation error of observable tensors,
which simplify the notations that arise in subsequent analysis.

\begin{deff}[canonical topic distribution]
Define the canonical version of topic distribution vector $\vct\mu_i$, $\widetilde{\vct\mu}_i$, as follows:
\begin{equation}
\widetilde{\vct\mu}_i \triangleq \sqrt{\frac{\alpha_i}{\alpha_0(\alpha_0+1)}}\vct\mu_i.
\label{def_can_mu}
\end{equation}

We also define $O,\widetilde O\in\mathbb R^{n\times k}$ by $O = [\vct\mu_1,\cdots,\vct\mu_k]$ and $\widetilde{O} = [\widetilde{\vct\mu_1},\cdots,\widetilde{\vct\mu_k}]$.
\end{deff}

\begin{deff}[estimation error]
Assume
\begin{eqnarray}
\|M_2-\widehat M_2\| &\leq& E_P,\\
\|M_y-\widehat M_y\| &\leq& E_y,\\
\|M_3-\widehat M_3\| &\leq& E_T.
\end{eqnarray}
for some real values $E_P,E_y$ and $E_T$, which we will set later.
\end{deff}

The following lemma analyzes the whitening matrix $W$ of $M_2$.
Many conclusions are directly from \cite{a:twosvd}.

\begin{lem}[Lemma C.1, \cite{speclda}]
Let $W, \widehat W\in\mathbb R^{n\times k}$ be the whitening matrices such that $M_2(W,W) = \widehat M_2(\widehat W,\widehat W) = I_k$. Let $A = W^\top\widetilde O$ and $\widehat A = \widehat W^\top\widetilde O$. Suppose $E_P \leq \sigma_k(M_2)/2$. We have
\begin{eqnarray}
\|W\| &=& \frac{1}{\sigma_k(\widetilde O)},\\
\|\widehat W\| &\leq& \frac{2}{\sigma_k(\widetilde O)},\label{eq6_pert}\\
\|W-\widehat W\| &\leq& \frac{4E_P}{\sigma_k(\widetilde O)^3},\label{eq_wwhat}\\
\|W^+\| &\leq& 3\sigma_1(\widetilde O),\\
\|\widehat W^+\| &\leq& 2\sigma_1(\widetilde O),\\
\|W^+ - \widehat W^+\| &\leq& \frac{6\sigma_1(\widetilde O)}{\sigma_k(\widetilde O)^2}E_P,\\
\|A\| &=& 1,\\
\|\widehat A\| &\leq & 2,\\
\|A - \widehat A\| &\leq& \frac{4E_P}{\sigma_k(\widetilde O)^2},\\
\|AA^\top - \widehat A\widehat A^\top\| &\leq& \frac{12E_P}{\sigma_k(\widetilde O)^2}.\label{eq4_pert}
\end{eqnarray}
\label{lem_whiten}
\end{lem}

\begin{proof}
Proof to Eq. (\ref{eq_wwhat}):
Let $\widehat W^\top\widehat M_2\widehat W = I$ and $\widehat W^\top M_2\widehat W = BDB^\top$, where $B$ is orthogonal and $D$ is a positive definite diagonal matrix.
We then see that $W = \widehat WBD^{-1/2}B^\top$ satisfies the condition $W M_2 W^\top = I$.
Subsequently, $\widehat W = WBD^{1/2}B^\top$.
We then can bound $\|W-\widehat W\|$ as follows
$$ \|W-\widehat W\| \leq \|W\|\cdot \|I-D^{1/2}\| \leq \|W\|\cdot \|I-D\| \leq \frac{4E_P}{\sigma_k(\widetilde O)^3},$$
where the inequality $\|I-D\|\leq \frac{4E_P}{\sigma_k(\widetilde O)^2}$ was proved in \cite{speclda}.

Proof to Eq. (\ref{eq4_pert}):
$\|AA^\top-\widehat A\widehat A^\top\| \leq \|AA^\top-A\widehat A^\top\| + \|A\widehat A^\top - \widehat A\widehat A^\top\|
\leq \|A-\widehat A\|\cdot (\|A\|+\|\widehat A\|) \leq \frac{12E_P}{\sigma_k(\widetilde O)^2}$.

All the other inequalities come from Lemma C.1, \cite{speclda}.
\end{proof}

We are now able to provide perturbation bounds for estimation error of whitened moments.

\begin{deff}[estimation error of whitened moments]
Define
\begin{eqnarray}
\varepsilon_{p,w} &\triangleq& \|M_2(W, W) - \widehat M_2(\widehat W,\widehat W)\|,\\
\varepsilon_{y,w} &\triangleq& \|M_y(W, W) - \widehat M_y(\widehat W,\widehat W)\|,\\
\varepsilon_{t,w} &\triangleq& \|M_3(W,W,W) - \widehat M_3(\widehat W,\widehat W,\widehat W)\|.
\end{eqnarray}
\label{def_epsilon}
\end{deff}
\begin{lem}[Perturbation lemma of whitened moments]
Suppose $E_P\leq \sigma_k(M_2)/2$. We have
\begin{eqnarray}
\varepsilon_{p,w} &\leq& \frac{16E_P}{\sigma_k(\widetilde O)^2},\\
\varepsilon_{y,w} &\leq& \frac{24\|\vct\eta\|E_P}{(\alpha_0+2)\sigma_k(\widetilde O)^2} + \frac{4E_y}{\sigma_k(\widetilde O)^2},\\
\varepsilon_{t,w} &\leq& \frac{54E_P}{(\alpha_0+1)(\alpha_0+2)\sigma_k(\widetilde O)^5} + \frac{8E_T}{\sigma_k(\widetilde O)^3}.
\end{eqnarray}
\end{lem}

\begin{proof}
Using the idea in the proof of Lemma C.2 in \cite{speclda}, we can split $\varepsilon_{p,w}$ as
\begin{equation*}
\begin{aligned}
	\varepsilon_{p,w}
	=& \|M_2(W,W) - M_2(\widehat W,\widehat W) + M_2(\widehat W,\widehat W) - \widehat M_2(\widehat W,\widehat W)\|\\
\leq& \|M_2(W,W) - M_2(\widehat W,\widehat W)\| + \|M_2(\widehat W,\widehat W) - \widehat M_2(\widehat W,\widehat W)\|.
\end{aligned}
\end{equation*}

We can the bound the two terms seperately, as follows.


For the first term, we have
\setlength\arraycolsep{1pt}\begin{eqnarray*}
\|M_2(W,W) - M_2(\widehat W,\widehat W)\|
&=& \|W^\top M_2 W-\widehat W^\top\widehat M_2\widehat W\|\\
&=& \|AA^\top - \widehat A\widehat A^\top\|\\
&\leq& \frac{12E_P}{\sigma_k(\widetilde O)^2}.
\end{eqnarray*}
where the last inequality comes from Eq. (\ref{eq4_pert}).

For the second term, we have
$$\|M_2(\widehat W,\widehat W)-\widehat M_2(\widehat W,\widehat W)\| \leq \|\widehat W\|^2\cdot \|M_2-\widehat M_2\| \leq \frac{4E_P}{\sigma_k(\widetilde O)^2},$$
where the last inequality comes from Eq. (\ref{eq6_pert}).

Similarly, $\varepsilon_{y,w}$ can be splitted as $\|M_y(W,W)-M_y(\widehat W,\widehat W)\|$ and $\|M_y(\widehat W,\widehat W)-\widehat M_y(\widehat W,\widehat W)\|$, which can be bounded separately. For the first term, we have
\begin{equation*}
\begin{aligned}
	\|M_y(W,W) - M_y(\widehat W,\widehat W)\|
	=& \|W^\top M_y W - \widehat W^\top M_y\widehat W\|\\
=& \frac{2}{\alpha_0+2}\|A\diag(\vct\eta)A^\top - \widehat A\diag(\vct\eta)\widehat A^\top\|\\
\leq& \frac{2\|\vct\eta\|}{\alpha_0+2}\cdot \|AA^\top - \widehat A\widehat A^\top\|\\
\leq& \frac{24\|\vct\eta\|}{(\alpha_0+2)\sigma_k(\widetilde O)^2}\cdot E_P.
\end{aligned}
\end{equation*}

For the second term, we have
$$\|M_y(\widehat W,\widehat W)-\widehat M_y(\widehat W,\widehat W)\| \leq \|\widehat W\|^2\cdot\|M_y-\widehat M_y\| \leq \frac{4E_y}{\sigma_k(\widetilde O)^2}.$$

Finally, we bound $\varepsilon_{t,w}$ as below, following the work \cite{specregression}.

\begin{eqnarray*}
\varepsilon_{t,w} &=& \|M_3(W,W,W) - \widehat M_3(\widehat W,\widehat W,\widehat W)\|\\
&\leq& \|M_3\|\cdot \|W-\widehat W\|\cdot (\|W\|^2 + \|W\|\cdot \|\widehat W\|
 + \|\widehat W\|^2) + \|\widehat W\|^3\cdot \|M_3-\widehat M_3\|\\
&\leq& \frac{54E_P}{(\alpha_0+1)(\alpha_0+2)\sigma_k(\widetilde O)^5} + \frac{8E_T}{\sigma_k(\widetilde O)^3},
\end{eqnarray*}
where we have used the fact that $$\|M_3\| \leq \sum_{i=1}^k{\frac{2\alpha_i}{\alpha_0(\alpha_0+1)(\alpha_0+2)}} = \frac{2}{(\alpha_0+1)(\alpha_0+2)}.$$

\end{proof}

\subsection{SVD accuracy}

The key idea for spectral recovery of LDA topic modeling is the \emph{simultaneous diagonalization} trick,
which asserts that we can recover LDA model parameters by performing orthogonal tensor decomposition
on a pair of simultaneously whitened moments, for example, $(M_2, M_3)$ and $(M_2, M_y)$.
The following proposition details this insight, as we derive orthogonal tensor decompositions for the whitened tensor product
$M_y(W,W)$ and $M_3(W,W,W)$.

\begin{prop}
Define $\vct v_i\triangleq W^\top \widetilde{\vct\mu}_i = \sqrt{\frac{\alpha_i}{\alpha_0(\alpha_0+1)}}W^\top\vct\mu_i$. Then
\begin{enumerate}
\item $\{\vct v_i\}_{i=1}^k$ is an orthonormal basis.
\item $M_y$ has a pair of singular value and singular vector $(\sigma_i^y, \vct v_i)$ with $\sigma_i^y = \frac{2}{\alpha_0+2}\eta_j$ for some $j \in[k]$.
\item $M_3$ has a pair of robust eigenvalue and eigenvector \cite{a:tensordecomp} $(\lambda_i, \vct v_i)$ with $\lambda_i = \frac{2}{\alpha_0+2}\sqrt{\frac{\alpha_0(\alpha_0+1)}{\alpha_{j'}}}$ for some $j'\in[k]$.
\end{enumerate}
\end{prop}

\begin{proof}
The orthonormality of $\{\vct v_i\}_{i=1}^k$ follows from the fact that $W^\top M_2W = \sum_{i=1}^k{\vct v_i\vct v_i^\top} = I_k$.
Subsequently, we have
\setlength\arraycolsep{1pt}\begin{eqnarray*}
M_y(W, W) &=& \frac{2}{\alpha_0+2}\sum_{i=1}^k{\eta_i \vct v_i\vct v_i^\top},\\
M_3(W,W,W) &=& \frac{2}{\alpha_0+2}\sum_{i=1}^k{\sqrt{\frac{\alpha_0(\alpha_0+1)}{\alpha_i}}\vct v_i\otimes\vct v_i\otimes\vct v_i}.
\end{eqnarray*}
\end{proof}

The following lemmas (Lemma \ref{lem_eta} and Lemma \ref{lem_mu})
give upper bounds on the estimation error of $\vct\eta$ and $\vct\mu$
in terms of $|\widehat\lambda_i-\lambda_i|$, $|\widehat{\vct v}_i-\vct v_i|$
and the estimation errors of whitened moments defined in Definition \ref{def_epsilon}.

\begin{lem}[$\eta_i$ estimation error bound]
Define $\widehat\eta_i \triangleq \frac{\alpha_0+2}{2}\widehat{\vct v}_i^\top\widehat M_y(\widehat W,\widehat W)\widehat{\vct v}_i$,
where $\widehat{\vct v}_i$ is some estimation of $\vct v_i$.
We then have
\begin{equation}
|\eta_i - \widehat\eta_i| \leq 2\|\vct\eta\|\|\widehat{\vct v}_i-\vct v_i\| + \frac{\alpha_0+2}{2}(1+2\|\widehat{\vct v}_i-\vct v_i\|)\cdot\varepsilon_{y,w}.
\end{equation}
\label{lem_eta}
\end{lem}

\begin{proof}
First, note that $\vct v_i^\top M_y(W,W)\vct v_i = \frac{2}{\alpha_0+2}\eta_i$ because $\{\vct v_i\}_{i=1}^k$ are orthonormal.
Subsequently, we have
\begin{equation*}
\begin{aligned}
	\frac{2}{\alpha_0+2}|\eta_i - \widehat\eta_i|
	=& \Big|\widehat{\vct v}_i^\top\widehat M_y(\widehat W,\widehat W)\widehat{\vct v}_i - \vct v_i^\top M_y(W,W)\vct v_i\Big| \\
\leq& \Big| (\widehat{\vct v}_i-\vct v_i)^\top \widehat M_y(\widehat W,\widehat W)\widehat{\vct v}_i\Big| + \Big|\vct v_i^\top\left(\widehat M_y(\widehat W,\widehat W)\widehat{\vct v}_i - M_y(W,W)\vct v_i\right)\Big|\\
\leq& \|\widehat{\vct v}_i-\vct v_i\|\|\widehat M_y(\widehat W,\widehat W)\|\|\widehat{\vct v}_i\| + \|\vct v_i\|\|\widehat M_y(\widehat W,\widehat W)\widehat{\vct v}_i - M_y(W,W)\vct v_i\|.\\
\end{aligned}
\end{equation*}

Note that both $\vct v_i$ and $\widehat{\vct v}_i$ are unit vectors. Therefore,
\begin{equation*}
\begin{aligned}
	\frac{2}{\alpha_0+2}|\eta_i-\widehat\eta_i|
\leq& \|\widehat M_y(\widehat W,\widehat W)\|\|\widehat{\vct v}_i-\vct v_i\| + \|\widehat M_y(\widehat W,\widehat W)\widehat{\vct v}_i - M_y(W,W)\vct v_i\|\\
	\leq& \|\widehat M_y(\widehat W,\widehat W)\|\|\widehat{\vct v}_i-\vct v_i\| + \|\widehat M_y(\widehat W,\widehat W)\|\|\widehat{\vct v}_i - \vct v_i\| + \|\widehat M_y(\widehat W,\widehat W)-M_y(W,W)\|\|\vct v_i\|\\
\leq& 2\|\widehat{\vct v}_i-\vct v_i\|\left(\frac{2}{\alpha_0+2}\|\vct\eta\| + \varepsilon_{y,w}\right) + \varepsilon_{y,w}.
\end{aligned}
\end{equation*}

The last inequality is due to the fact that $\|M_y(W,W)\| = \frac{2}{\alpha_0+2}\|\vct\eta\|$.

\end{proof}

\begin{lem}[$\vct\mu_i$ estimation error bound]
Define $\widehat{\vct\mu}_i \triangleq \frac{\alpha_0+2}{2}\widehat\lambda_i(\widehat W^+)^\top\widehat{\vct v}_i$,
where $\widehat\lambda_i, \widehat{\vct v}_i$ are some estimates of singular value pairs $(\lambda_i, \vct v_i)$ of $M_3(W,W,W)$.
We then have
\begin{equation}
\begin{aligned}
	\|\widehat{\vct\mu}_i - \vct\mu_i\| \leq& \frac{3(\alpha_0+2)}{2}\sigma_1(\widetilde O)|\widehat\lambda_i-\lambda_i|
											+ 3\alpha_{\max}\sigma_1(\widetilde O)\|\widehat{\vct v}_i-\vct v_i\|
		+ \frac{6\alpha_{\max}\sigma_1(\widetilde O)E_P}{\sigma_k(\widetilde O)^2}.
\end{aligned}
\end{equation}
\label{lem_mu}
\end{lem}

\begin{proof}
First note that $\vct\mu_i = \frac{\alpha_0+2}{2}\lambda_i (W^+)^\top\vct v_i$. Subsequently,
\begin{equation*}
\begin{aligned}
	\frac{2}{\alpha_0+2}\|\vct\mu_i - \widehat{\vct\mu}_i\|
=& \|\widehat\lambda_i(\widehat W^+)^\top\widehat{\vct v}_i - \lambda_i(W^+)^\top\vct v_i\|\\
\leq& \|\widehat\lambda_i\widehat W^+ - \lambda_i W^+\|\|\widehat{\vct v}_i\| + \|\lambda_i W^+\|\|\widehat{\vct v}_i-\vct v_i\|\\
\leq& |\widehat\lambda_i-\lambda_i|\|\widehat W^+\| + |\lambda_i|\|\widehat W^+ - W^+\| + |\lambda_i|\|W^+\|\|\widehat{\vct v}_i-\vct v_i\|\\
\leq& 3\sigma_1(\widetilde O)|\widehat\lambda_i-\lambda_i| + \frac{2\alpha_{\max}}{\alpha_0+2}\cdot \frac{6\sigma_1(\widetilde O) E_P}{\sigma_k(\widetilde O)^2}
	+ \frac{2\alpha_{\max}}{\alpha_0+2}\cdot 3\sigma_1(\widetilde O)\cdot \|\widehat{\vct v_i} - \vct v_i\|.
\end{aligned}
\end{equation*}
\end{proof}

To bound the error of orthogonal tensor decomposition performed on the estimated tensors $\widehat M_3(\widehat W,\widehat W,\widehat W)$,
we cite Theorem 5.1 \cite{a:tensordecomp}, a sample complexity analysis on the robust tensor power method we used
for recovering $\widehat\lambda_i$ and $\widehat{\vct v}_i$.

\begin{lem}[Theorem 5.1, \cite{a:tensordecomp}]
Let $\lambda_{\max} = \frac{2}{\alpha_0+2}\sqrt{\frac{\alpha_0(\alpha_0+1)}{\alpha_{\min}}}$,
$\lambda_{\min} = \frac{2}{\alpha_0+2}\sqrt{\frac{\alpha_0(\alpha_0+1)}{\alpha_{\max}}}$,
where $\alpha_{\min} = \min{\alpha_i}$ and $\alpha_{\max} = \max{\alpha_i}$.
Then there exist universal constants $C_1, C_2 > 0$ such that the following holds:
Fix $\delta'\in(0,1)$.
Suppose $\varepsilon_{t,w}\leq \varepsilon$ and
\begin{eqnarray}
\varepsilon_{t, w} &\leq& C_1\cdot \frac{\lambda_{\min}}{k},\label{eq1_tensorpower}
\end{eqnarray}
Suppose $\{(\widehat\lambda_i, \widehat{\vct v}_i)\}_{i=1}^k$ are eigenvalue and eigenvector pairs
returned by running Algorithm 1 in \cite{a:tensordecomp} with input $\widehat M_3(\widehat W,\widehat W,\widehat W)$ for $L = \text{poly}(k)\log(1/\delta')$ and $N \geq C_2\cdot(\log(k) + \log\log(\frac{\lambda_{\max}}{\varepsilon}))$ iterations.
With probability greater than $1-\delta'$, there exists a permutation $\pi':[k]\to[k]$ such that
for all $i$,
$$\|\widehat{\vct v}_i - \vct v_{\pi'(i)}\| \leq 8\varepsilon/\lambda_{\min},\quad |\widehat\lambda_i - \lambda_{\pi'(i)}| \leq 5\varepsilon.$$
\label{lem_tensorpower}
\end{lem}

\subsection{Tail Inequalities}

\begin{lem}[Lemma 5, \cite{specregression}]
Let $\vct x_1,\cdots,\vct x_N\in\mathbb R^d$ be i.i.d. samples from some distribution with bounded support (i.e., $\|\vct x\|_2\leq B$ with probability 1 for some constant $B$).
Then with probability at least $1-\delta$,
$$\left\|\frac{1}{N}\sum_{i=1}^N{\vct x_i} - \mathbb E[\vct x]\right\|_2 \leq \frac{2B}{\sqrt{N}}\left(1+\sqrt{\frac{\log(1/\delta)}{2}}\right).$$
\label{lem_chernoff}
\end{lem}

\begin{cor}
Let $\vct x_1,\cdots,\vct x_N\in\mathbb R^d$ be i.i.d. samples from some distributions with $\Pr[\|\vct x\|_2 \leq B] \geq 1-\delta'$.
Then with probability at least $1-N\delta'-\delta$,
$$\left\|\frac{1}{N}\sum_{i=1}^N{\vct x_i} - \mathbb E[\vct x]\right\|_2 \leq \frac{2B}{\sqrt{N}}\left(1+\sqrt{\frac{\log(1/\delta)}{2}}\right).$$
\label{cor_chernoff}
\end{cor}
\begin{proof}
Use union bound.
\end{proof}

\begin{lem}[concentration of moment norms]
Suppose we obtain $N$ i.i.d. samples (i.e., documents with at least three words each and their regression variables in sLDA models).
Define $R(\delta) \triangleq \|\vct\eta\| -  \sigma\Phi^{-1}(\delta)$,
where $\Phi^{-1}(\cdot)$ is the inverse function of the CDF of a standard Gaussian distribution.
Let $\mathbb E[\cdot]$ denote the mean of the true underlying distribution
and $\widehat{\mathbb E}[\cdot]$ denote the empirical mean.
Then

\begin{equation}
\begin{aligned}
	& \Pr \left [\|\mathbb E[\vct x_1] - \widehat{\mathbb E}[\vct x_1]\|_F < \frac{2+\sqrt{2\log(1/\delta)}}{\sqrt{N}} \right ] \geq 1-\delta,\\
	&\Pr \left [\|\mathbb E[\vct x_1\otimes\vct x_2] - \widehat{\mathbb E}[\vct x_1\otimes\vct x_2]\|_F < \frac{2+\sqrt{2\log(1/\delta)}}{\sqrt{N}} \right ] \geq 1-\delta,\\
	&\Pr \left [\|\mathbb E[\vct x_1\otimes\vct x_2\otimes\vct x_3] - \widehat{\mathbb E}[\vct x_1\otimes\vct x_2\otimes\vct x_3]\|_F< \frac{2+\sqrt{2\log(1/\delta)}}{\sqrt{N}} \right ] \geq 1-\delta,\\
&\Pr \left [\|\mathbb E[y] - \widehat{\mathbb E}[y]\| < R(\delta/4 N)\cdot\frac{2+\sqrt{2\log(2/\delta)}}{\sqrt{N}} \right ] \geq 1-\delta,\\
&\Pr \left [\|\mathbb E[y\vct x_1] - \widehat{\mathbb E}[y\vct x_1]\|_F < R(\delta/4 N)\cdot\frac{2+\sqrt{2\log(2/\delta)}}{\sqrt{N}} \right ] \geq 1-\delta,\\
&\Pr \left [\|\mathbb E[y\vct x_1\otimes\vct x_2] - \widehat{\mathbb E}[y\vct x_1\otimes\vct x_2]\|_F  < R(\delta/4 N)\cdot\frac{2+\sqrt{2\log(2/\delta)}}{\sqrt{N}} \right ] \geq 1 - \delta.
\end{aligned}
\end{equation}
\label{lem_tail_ineq}
\end{lem}

\begin{proof}
Use Lemma \ref{lem_chernoff} and Corrolary \ref{cor_chernoff} for concentration bounds involving the regression variable $y$.
\end{proof}

\begin{cor}
With probability $1-\delta$ the following holds:
\begin{enumerate}
\item $E_P = \|M_2-\widehat M_2\| \leq 3\cdot \frac{2+\sqrt{2\log(6/\delta)}}{\sqrt{N}}$.
\item $E_y = \|M_y-\widehat M_y\| \leq 10R(\delta/60 N)\cdot \frac{2+\sqrt{2\log(15/\delta)}}{\sqrt{N}}$.
\item $E_T = \|M_3-\widehat M_3\| \leq 10\cdot \frac{2+\sqrt{2\log(9/\delta)}}{\sqrt{N}}$.
\end{enumerate}
\label{cor_ep}
\end{cor}

\begin{proof}
Corrolary \ref{cor_ep} can be proved by expanding the terms by definition and then using tail inequality in Lemma \ref{lem_tail_ineq} and union bound.
Also note that $\|\cdot\| \leq \|\cdot\|_F$ for all matrices.
\end{proof}

\subsection{Completing the proof}

We are now ready to give a complete proof to Theorem \ref{thm_main}.

\begin{proof}{(Proof of Theorem ~\ref{thm_main})}
First, the assumption $E_P\leq\sigma_k(M_2)$ is required for error bounds on $\varepsilon_{p,w}, \varepsilon_{y,w}$ and $\varepsilon_{t,w}$.
Noting Corrolary \ref{cor_ep} and the fact that $\sigma_k(M_2) = \frac{\alpha_{\min}}{\alpha_0(\alpha_0+1)}$, we have
$$N = \Omega\left(\frac{\alpha_0^2(\alpha_0+1)^2(1+\sqrt{\log(6/\delta)})^2}{\alpha_{\min}^2}\right).$$
Note that this lower bound does not depend on $k$, $\varepsilon$ and $\sigma_k(\widetilde O)$.

For Lemma \ref{lem_tensorpower} to hold, we need the assumptions that
$\varepsilon_{t,w} \leq \min(\varepsilon, O(\frac{\lambda_{\min}}{k}))$.
These imply $n_3$,
as we expand $\varepsilon_{t,w}$ according to Definition \ref{def_epsilon}
and note the fact that the first term $\frac{54E_P}{(\alpha_0+1)(\alpha_0+2)\sigma_k(\widetilde O)^5}$ dominates the second one.
The $\alpha_0$ is missing in the third requirment $n_3$ because $\alpha_0+1 \geq 1$, $\alpha_0+2 \geq 2$ and we discard them both.
The $|\alpha_i-\widehat{\alpha}_{\pi(i)}|$ bound follows immediately by Lemma \ref{lem_tensorpower}
and the recovery rule $\widehat{\alpha}_i = \frac{\alpha_0+2}{2}\widehat\lambda_i$.

To bound the estimation error for the linear classifier $\vct\eta$, we need to further bound $\varepsilon_{y,w}$.
We assume $\varepsilon_{y,w}\leq\varepsilon$.
By expanding $\varepsilon_{y,w}$ according to Definition \ref{def_epsilon} in a similar manner
we obtain the $(\|\vct\eta\|+\Phi^{-1}(\delta/60\sigma))^2$ term in the requirment of $n_2$.
The bound on $|\eta_i-\widehat\eta_{\pi(i)}|$ follows immediately by Lemma \ref{lem_eta}.

Finally, we bound $\|\vct\mu_i-\widehat{\vct\mu}_{\pi(i)}\|$ using Lemma \ref{lem_mu}.
We need to assume that $\frac{6\alpha_{\max}\sigma_1(\widetilde O)E_P}{\sigma_k(\widetilde O)^2} \leq \varepsilon$,
which gives the $\alpha_{\max}^2\sigma_1(\widetilde O)^2$ term. 
The $\|\vct\mu_i-\widehat{\vct\mu}_{\pi(i)}\|$ bound then follows by Lemma \ref{lem_mu} and Lemma \ref{lem_tensorpower}.

\end{proof}

We make some remarks for the main theorem. In Remark~\ref{rem_connection}, we establish links between indirect quantities appeared in Theorem~\ref{thm_main} (e.g., $\sigma_k(\widetilde O)$) and the functions of original model parameters (e.g., $\sigma_k(O)$).
These connections are straightforward following their definitions.

\begin{rem}
The indirect quantities $\sigma_1(\widetilde O)$ and $\sigma_k(\widetilde O)$ can be related to $\sigma_1(O)$, $\sigma_k(O)$ and $\vct\alpha$ in the following way:
{\small \begin{equation*}
\sqrt{\frac{\alpha_{\min}}{\alpha_0(\alpha_0+1)}}\sigma_k(O) \leq \sigma_k(\widetilde O)\leq \sqrt{\frac{\alpha_{\max}}{\alpha_0(\alpha_0+1)}}\sigma_k(O);~~~~
\end{equation*}
\begin{equation*}
\sigma_1(\widetilde O) \leq \sqrt{\frac{\alpha_{\max}}{\alpha_0(\alpha_0+1)}}\sigma_1(O) \leq \frac{1}{\sqrt{\alpha_0+1}}.
\end{equation*}}
\label{rem_connection}
\end{rem}

We now take a close look at the sample complexity bound in Theorem~\ref{thm_main}.
It is evident that $n_2$ can be neglected when the number of topics $k$ gets large,
because in practice the norm of the linear regression model $\vct\eta$ is usually assumed to be small in order to avoid overfitting.
Moreover, as mentioned before, the prior parameter $\vct\alpha$ is often assumed to be homogeneous with $\alpha_i = 1/k$ \cite{lsa}.
With these observations, the sample complexity bound in Theorem~\ref{thm_main}
can be greatly simplified. 
\begin{rem}
Assume $\|\vct\eta\|$ and $\sigma$ are small and $\vct\alpha = (1/k, \cdots, 1/k)$.
As the number of topics $k$ gets large, the sample complexity bound in Theorem~\ref{thm_main} can be simplified as
\begin{equation}
N = \Omega\left( \frac{\log(1/\delta)}{\sigma_k(\widetilde O)^{10}}\cdot\max(\varepsilon^{-2}, k^3)\right).
\end{equation}\label{rem_simpbound}
\end{rem}\vspace{-.5cm}
The sample complexity bound in Remark~\ref{rem_simpbound} may look formidable
as it depends on $\sigma_k(\widetilde O)^{10}$.
However, such dependency is somewhat necessary because we are using third-order tensors to recover the underlying model parameters.
Furthermore, the dependence on $\sigma_k(\widetilde O)^{10}$ is introduced by the robust tensor power method to recover LDA parameters,
and the reconstruction accuracy of $\vct\eta$ only depends on $\sigma_k(\widetilde O)^4$ and $(\|\vct\eta\|+\Phi^{-1}(\delta/60\sigma))^2$.
As a consequence, if we can combine our power update method for $\vct \eta$ with LDA inference algorithms that have milder dependence on the singular value $\sigma_k(\widetilde O)$,
we might be able to get an algorithm with a better sample complexity.

%% file: supplementary2.tex
\section{Proof of Theorem 2}
In this section we give the proof of Theorem 2, following the similar line of the proof of Theorem
1. We bound the estimation errors and the errors introduced by the tensor decomposition step
respectively.

\subsection{Definitions}
We first recall the definition of joint canconical topic distribution.
\begin{deff}
	(Joint Canconical topic distribution) Define the canconical version of joint topic distribution vector
	$\bm{v}_i,\widetilde{\bm{v}}_i$ as follows:
	$$
	\widetilde{\bm{v}}_i = \sqrt{\dfrac{\alpha_i}{\alpha_0(\alpha_0 +1)}}\bm{v}_i,
	$$
	where $\bm{v}_i = [ \bm{\mu}_i' , \eta_i]'$ is the topic dictribution vector extended by its regression parameter.

	We also define $O^*, \widetilde{O^*} \in \mathbb{R}^{(V+1) \times k}$ by $ O^* = [ \bm{v}_1, ... ,\bm{v}_k] $
	and $ \widetilde{O^*} = [ \widetilde{\bm{v}}_1 , ..., \widetilde{\bm{v}}_k]$.
		
\end{deff}

\subsection{Estimation Errors}
The tail inequatlities in last section \ref{lem_chernoff},\ref{cor_chernoff} gives the following estimation:
\begin{lem}
	\label{l1}
	Suppose we obtain $N$ i.i.d. samples. Let $\mathbb{E}[\cdot]$ denote the mean of the true underlying distribution and $\hat{\mathbb{E}}[\cdot]$ denote
	the empirical mean. Define
	\begin{equation}
	\begin{aligned}
		& R_1(\delta) = \|\bm{\eta}\| - \sigma\Phi^{-1}(\delta),\\
		& R_2(\delta) = 2\|\bm{\eta}\|^2 + 2\sigma^2[\Phi^{-1}(\delta)]^2, \\
		& R_3(\delta) = 4\|\bm{\eta}\|^3 - 4\sigma^3[\Phi^{-1}(\delta)]^3, \\
	\end{aligned}
	\end{equation}
	where $\Phi^{-1}(\cdot)$ is the inverse function of the CDF of a
	standard Gaussian distribution. Then:
	\begin{equation}
	\begin{aligned}
		&	\Pr \left [ \| \mathbb{E}[ \mathbf{x_1} \otimes \mathbf{x_2} \otimes \mathbf{x_3}]  - \hat{\mathbb{E}}[\mathbf{x_1} \otimes
\mathbf{x_2} \otimes \mathbf{x_3} ] \|_F   < \dfrac{2 + \sqrt{2\log(1/\delta)} }{ \sqrt{N}}  \right ]  \geq 1 - \delta; \\
												   &		\Pr \left [ \| \mathbb{E}[ y^i ]  - \hat{\mathbb{E}}[ y^i ] \|_F < R_i(\delta/4 N )\cdot
			 \dfrac{2 + \sqrt{2\log(2/\delta)} }{ \sqrt{N} } \right ]  \geq 1 - \delta  \ \ \ \ i = 1, 2, 3;\\
			 & 		\Pr \left [ \| \mathbb{E}[  y^i \mathbf{x_1} ]  - \hat{\mathbb{E}}[ y^i \mathbf{x_1}  ] \|_F
				 < R_i(\delta / 4 N) \cdot \dfrac{2 + \sqrt{2\log(2/\delta)} }{ \sqrt{N}} \right ]  \geq 1 - \delta  \ \ \ \ i = 1,2;\\
				 & 		\Pr \left [ \| \mathbb{E}[  y^i \mathbf{x_1} \otimes \mathbf{x_2} ]  - \hat{\mathbb{E}}[ y^i \mathbf{x_1} \otimes \mathbf{x_2} ] \|_F
				 < R_i(\delta / 4 N) \cdot \dfrac{2 + \sqrt{2\log(2/\delta)} }{ \sqrt{N}} \right ]  \geq 1 - \delta  \ \ \ \ i = 1 , 2.\\
	\end{aligned}
	\end{equation}
\end{lem}
\begin{proof}
	This lemma is a direct application of lemma \ref{lem_chernoff} and corollary \ref{cor_chernoff}.
\end{proof}

\begin{cor}
\label{c1}
	With probability $1 -\delta$, the following holds:
	\begin{equation}
	\begin{aligned}
		& Pr  \left [ \| \mathbb{E}[ \mathbf{z_1}] - \hat{\mathbb{E}}[ \mathbf{z_1}] \| < C_1(\delta / 8 N)
			\dfrac{2 + \sqrt{2\log(4/\delta)} }{ \sqrt{N}}   \right ] \geq 1 - \delta,   \\
			& Pr  \left [ \| \mathbb{E}[ \mathbf{z_1} \otimes \mathbf{z_2}] - \hat{\mathbb{E}}[ \mathbf{z_1} \otimes \mathbf{z_2}] \|   < C_2(\delta / 12 N)
		\dfrac{2 + \sqrt{2\log(6/\delta)} }{ \sqrt{N}}   \right ]   \geq 1 - \delta,   \\
		&	Pr  \left [ \| \mathbb{E}[ \mathbf{z_1} \otimes \mathbf{z_2} \otimes \mathbf{z_3}]  - \hat{\mathbb{E}}[\mathbf{z_1} \otimes
	\mathbf{z_2} \otimes \mathbf{z_3} ] \|_F   < C_3(\delta/16N) \dfrac{2 + \sqrt{2\log(8/\delta)} }{ \sqrt{N}}   \right ]  \geq 1 - \delta, \\
	\end{aligned}
	\end{equation}
	where
	\begin{equation}
	\begin{aligned}
		& C_1(\delta) = R_1(\delta) + 1, \\
		& C_2(\delta) = R_1(\delta) + R_2(\delta) + 1, \\
		& C_3(\delta) = R_1(\delta) + R_2(\delta) + R_3(\delta) + 1. \\
	\end{aligned}
	\end{equation}
\end{cor}
\begin{proof}
	Use lemma \ref{l1} as well as the fact that $\sqrt{a + b} \leq \sqrt{a} + \sqrt{b}, \forall a,b \geq 0$ and $ Pr( X \leq t_1 , Y \leq t_2) \leq Pr( X + Y \leq t_1 + t_2)$.
\end{proof}

\begin{lem}
\label{l3}
	(concentration of moment norms) Using notations in lemma \ref{l1} and suppose
	$C_3(\delta/144N)\dfrac{2 +\sqrt{2\log(72/\sigma)}}{\sqrt{N}} \leq 1 $
	we have:
	\begin{equation}
	\begin{aligned}
		& E_P = \|N_2 - \hat{N}_2\| \leq 3C_2(\delta/36N)) \cdot 	\dfrac{2 + \sqrt{2 \log(18/\sigma)}}{\sqrt{N}},\\
		& E_T = \|N_3 - \hat{N}_3\| \leq 10C_3(\delta/144N)) \cdot  \dfrac{2 + \sqrt{2 \log(72/\sigma)}}{\sqrt{N}}.\\
	\end{aligned}
	\end{equation}
\end{lem}
\begin{proof}
	This lemma can be proved by expanding the terms by definition and use lemma \ref{l1} , corollary \ref{c1} as well as union bound.
	Also note that $\| \cdot\| \leq \|\cdot \|_F$.
\end{proof}

\begin{lem}
	(estimation error of whitened moments)
	Define
	\begin{equation}
	\begin{aligned}
		& \epsilon_{p,w} = \| N_2(W,W) - \hat{N}_2(\hat{W},\hat{W})\|, \\
		& \epsilon_{t,w} = \| N_3(W,W,W) - \hat{N}_3(\hat{W},\hat{W},\hat{W})\|. \\
	\end{aligned}
	\end{equation}
	Then we have:
	\begin{equation}
	\begin{aligned}
		& \epsilon_{p,w} \leq  \dfrac{16E_p}{\sigma_k(\widetilde{O^*})^2}, \\
		& \epsilon_{t,w} \leq  \dfrac{54E_P}{(\alpha_0+1)(\alpha_0+2)\sigma_k(\widetilde{O^*})^5} +
		\dfrac{8E_T}{\sigma_k(\widetilde{O^*})^3}.
	\end{aligned}
	\end{equation}
\end{lem}

\subsection{SVD accuracy}

We rewrite a part of Lemma ~\ref{lem_whiten} which we will use in the following.

\begin{lem}
\label{l5}
	Let $W,\hat{W} \in R^{(V+1)\times k}$ be the whitening matrices such that
		$N_2(W,W) = \hat{W}_2(\hat{W},\hat{W}) = I_k$. Further more, we suppose
		$ E_P \leq \sigma_k(N_2)/2$. Then we have:
		\begin{equation}
		\begin{aligned}
			& \|W - \hat{W}\| \leq \dfrac{4E_P}{\sigma_k(\widetilde{O^*})^3},\\
			& \|W^+\| \leq 3\sigma_1(\widetilde{O^*}), \\
			& \|\widetilde{W}^+\| \leq 2\sigma_1(\widetilde{O^*}), \\
			& \|W^+ - \hat{W}^+\| \leq \dfrac{6\sigma_1(\widetilde{O^*})E_P}{\sigma_k(\widetilde{O^*})^2}.\\
		\end{aligned}
		\end{equation}
\end{lem}

Using the above lemma, we now can estimate the error introduced by SVD.

\begin{lem}($\bm{v}_i$ estimation error)
\label{l6}
	Define $\hat{\bm{v}}_i = \frac{\alpha_0 + 2}
	{2}\hat{\lambda}_i(\hat{W}^+)^{\top}\hat{\bm{\omega}}_i$ where
	 $(\hat{\lambda}_i,\hat{\omega}_i)$ are some estimations of svd pairs
	 $(\lambda_i, \omega_i)$ of $ N_3(W,W,W)$. We then have:
	\begin{equation}
	 \begin{aligned}
		 \|\hat{\bm{v}}_i - \bm{v}_i\| \leq& \dfrac{(\alpha_0 + 1)6\sigma_1(\widetilde{O^*})E_P}{\sigma_k(\widetilde{O^*})^2}
		 + \dfrac{(\alpha_0 +2)\sigma_1(\widetilde{O^*})|\hat{\lambda}_i - \lambda_i|}{2}
	 +3(\alpha_0 +1)\sigma_1(\widetilde{O^*})\|\hat{\bm{\omega}}_i - \bm{\omega}_i\|.
	\end{aligned}
	 \end{equation}
\end{lem}

\begin{proof}
	Note that $\bm{v}_i = \frac{}{}\lambda_i(W^+)^{\top}\bm{\omega}_i$. Thus we have:
	\begin{equation}
	\begin{aligned}
		\frac{2}{\alpha_0 + 2 } \|\hat{\bm{v}}_i - \bm{v}_i\|
		= &\|\hat{\lambda}_i(\hat{W}^+)^{\top}\hat{\bm{\omega}_i}
		- \lambda_i(W^+)^{\top}\bm{\omega}_i \| \\
		=& \|\hat{\lambda}_i(\hat{W}^+)^{\top}\hat{\bm{\omega}_i}
		- \lambda_i(W^+)^{\top}\hat{\bm{\omega}}_i
		+ \lambda_i(W^+)^{\top}\hat{\bm{\omega}}_i
		- \lambda_i(W^+)^{\top}\bm{\omega}_i \| \\
		\leq& \|\hat{\lambda}_i(\hat{W}^+)^{\top} - \lambda_i(W^+)^{\top}\|
			\|\hat{\bm{\omega}}_i\|
			+ \|\lambda_i(W^+)^{\top}\|\|\hat{\bm{\omega}}_i - \bm{\omega}_i\|\\
			\leq& |\lambda_i|\|\hat{W}^+ - W^+\|
			\|\hat{\bm{\omega}}_i\| +
			|\hat{\lambda}_i - \lambda_i|\|\hat{W}^+\|
			\|\hat{\bm{\omega}}_i\|
			+ |\lambda_i|\|W^+\|\|\hat{\bm{\omega}}_i - \bm{\omega}_i\|\\
			\leq& \dfrac{2(\alpha_0 +1)}{\alpha_0 + 2}\dfrac{6E_P\sigma_1(\widetilde{O^*})}{\sigma_k(\widetilde{O^*})^2}
			+ 3\sigma_1(\widetilde{O^*})|\hat{\lambda}_i - \lambda_i| +
			\dfrac{2(\alpha_0 + 1)}{\alpha_0 + 2}\cdot
			3\sigma_1(\widetilde{O^*})\cdot \|\hat{\bm{\omega}}_i - \bm{\omega}_i \|.
	\end{aligned}
	\end{equation}
\end{proof}

\subsection{Completeing the proof}
We are now ready to complete the proof of Theorem ~\ref{thm_main2}. 

\begin{proof}{(Proof of Theorem~\ref{thm_main2})}
	We check out the conditions that must be satisfied for $\epsilon$ accuracy.
	First $E_P \leq \sigma_k(N_2)/2$ is required in lemma \ref{l5}. Noting that $\sigma_k(N_2) = \dfrac{\alpha_{min}}{\alpha_0(\alpha_0 + 1)} $, we have:
	$$
		N \geq O\left (  \dfrac{\alpha_0^2(\alpha_0+1)^2C_2^2(\delta/36N)\cdot(2 + \sqrt{2\log(18/\delta)})^2 }{\alpha_{min}^2}\right ).
	$$
	In lemma \ref{l3}, we need that $C_3(\delta/144N)\dfrac{2\log(72/\sigma)}{\sqrt{N}} \leq 1$, which means:
	$$
		N \geq O \left ( C_3^2(\delta/144N)(2 + \sqrt{2 \log(72/\sigma)})^2 \right ).
	$$
	For lemma \ref{lem_tensorpower} to hold, we need the assumption that $\epsilon_{t,w} \leq C_1 \cdot \frac{\lambda_{min}}{k}$, which implies:
	$$
		N \geq 	O \left ( \dfrac{C_3^2(\delta/36N)(2+\sqrt{2\log(18/\sigma)})^2}{\sigma_k(\widetilde{O^*})^{10}}\cdot\max(\dfrac{1}{\epsilon^2},
		\dfrac{k^2}{\lambda_{min}^2}) \right ).
	$$

	From the recovery rule $\hat{\alpha_i} = \dfrac{4\alpha_0(\alpha_0 + 1)}{(\alpha_0 + 2)^2 \hat{\lambda_i}^2}$, we have:
	\begin{equation}
	\begin{aligned}
		 |\hat{\alpha}_i - \alpha_i |
		  \leq \dfrac{4\alpha_0(\alpha_0 + 1) }{(\alpha_0 + 2)^2}
		\left |\dfrac{1}{\hat{\lambda}_i^2} - \dfrac{1}{\lambda_i^2} \right |
		\leq \dfrac{4\alpha_0(\alpha_0 + 1) }{(\alpha_0 + 2)^2\lambda_{min}^2(\lambda_{min} - 5 \epsilon)^2} \cdot 5\epsilon. \\
	\end{aligned}
	\end{equation}
	Finally, we bound $\|\hat{\bm{v}}_i - \bm{v}_i\|$, which is a direct consequence of lemma \ref{l6} and lemma \ref{lem_tensorpower}.
	We additional assume the following condition holds: $\dfrac{6E_P}{\sigma_k(\widetilde{O^*})^2} \leq \epsilon $. In fact, this condition is already satisfied
	if the above requirements for $N$ hold.
	Under this condition, we have:
	$$
	\|\hat{\bm{v}}_i - \bm{v}_i \|	\leq \left ( \sigma_1(\widetilde{O^*})(\alpha_0  + 2)(\frac{7}{2} +
	\frac{8}{\lambda_{min}})\right )\cdot \epsilon.\\
	$$
\end{proof}

%% file: supplementary3.tex
\section{Moments of Observable Variables}

\subsection{Proof to Proposition 1}

The equations on $M_2$ and $M_3$ have already been proved in \cite{speclda} and \cite{a:tensordecomp}.
Here we only give the proof to the equation on $M_y$.
In fact, all the three equations can be proved in a similar manner.

In sLDA the topic mixing vector $\vct h$ follows a Dirichlet prior distribution with parameter $\vct\alpha$.
Therefore, we have
\begin{equation}
\begin{aligned}
	&\mathbb E[h_i] = \frac{\alpha_i}{\alpha_0},\\
 &\mathbb E[h_ih_j] = \left\{\begin{array}{ll} \frac{\alpha_i^2}{\alpha_0(\alpha_0+1)}, &\text{if }i=j,\\
\frac{\alpha_i\alpha_j}{\alpha_0^2}, &\text{if }i\neq j\end{array}\right.,\\
									 &\mathbb E[h_ih_jh_k] = \left\{\begin{array}{ll} \frac{\alpha_i^3}{\alpha_0(\alpha_0+1)(\alpha_0+2)}, & \text{if }i=j=k,\\
\frac{\alpha_i^2\alpha_k}{\alpha_0^2(\alpha_0+1)}, & \text{if }i=j\neq k,\\
\frac{\alpha_i\alpha_j\alpha_k}{\alpha_0^3}, &\text{if }i\neq j, j\neq k, i\neq k\end{array}\right. .
\end{aligned}
\end{equation}

Next, note that
\begin{equation}
\begin{aligned}
	&\mathbb E[y|\vct h] = \vct\eta^\top\vct h,\\
 &\mathbb E[\vct x_1|\vct h] = \sum_{i=1}^k{h_i\vct\mu_i},\\
 &\mathbb E[\vct x_1\otimes\vct x_2|\vct h] = \sum_{i, j=1}^k{h_ih_j\vct\mu_i\otimes\vct\mu_j},\\
 &\mathbb E[y\vct x_1\otimes\vct x_2|\vct h] = \sum_{i, j, k=1}^k{h_ih_jh_k\cdot \eta_k\vct\mu_j\otimes\vct\mu_k}.
\end{aligned}
\end{equation}

Proposition 1 can then be proved easily by taking expectation over the topic mixing vector $\vct h$.

\subsection{Details of the speeding-up trick}

In this section we provide details of the trick mentioned in the main paper to speed up empirical moments computations.
First, note that the computation of $\widehat M_1$, $\widehat M_2$ and $\widehat M_y$ only requires $O(NM^2)$ time and $O(V^2)$ space.
They do not need to be accelerated in most practical applications.
This time and space complexity also applies to all terms in $\widehat M_3$ except the $\widehat{\mathbb E}[\vct x_1\otimes\vct x_2\otimes\vct x_3]$ term,
which requires $O(NM^3)$ time and $O(V^3)$ space if using naive implementations.
Therefore, this section is devoted to speed-up the computation of $\widehat{\mathbb E}[\vct x_1\otimes\vct x_2\otimes\vct x_3]$.
More precisely, as mentioned in the main paper,
what we want to compute is the whitened empirical moment $\widehat{\mathbb E}[\vct x_1\otimes\vct x_2\otimes\vct x_3](\widehat W,\widehat W,\widehat W) \in
\mathbb R^{k\times k\times k}$.

Fix a document $D$ with $m$ words.
Let $T \triangleq \widehat{\mathbb E}[\vct x_1\otimes\vct x_2\otimes\vct x_3|D]$ be the empirical tensor demanded.
By definition, we have

\begin{equation}
T_{i,j,k} = \frac{1}{m(m-1)(m-2)}\left\{\begin{array}{ll}
n_i(n_j-1)(n_k-2),& i=j=k;\\
n_i(n_i-1)n_k,& i=j,j\neq k;\\
n_in_j(n_j-1),& j=k,i\neq j;\\
n_in_j(n_i-1),& i=k,i\neq j;\\
n_in_jn_k,& \text{otherwise};\end{array}\right.
\label{eq_T}
\end{equation}
where $n_i$ is the number of occurrences of the $i$-th word in document $D$.
If $T_{i,j,k} = \frac{n_in_jn_k}{m(m-1)(m-2)}$ for all indices $i,j$ and $k$,
then we only need to compute
$$T(W,W,W) = \frac{1}{m(m-1)(m-2)}\cdot (W\vct n)^{\otimes 3},$$
where $\vct n \triangleq (n_1,n_2,\cdots,n_V)$.
This takes $O(Mk + k^3)$ computational time because $\vct n$ contains at most $M$ non-zero entries,
and the total time complexity is reduced from $O(NM^3)$ to $O(N(Mk+k^3))$.

We now consider the remaining values, where at least two indices are identical.
We first consider those values with two indices the same, for example, $i=j$.
For these indices, we need to subtract an $n_in_k$ term, as shown in Eq. (\ref{eq_T}).
That is, we need to compute the whitened tensor $\Delta(W,W,W)$, where $\Delta\in\mathbb R^{V\times V\times V}$ and
\begin{equation}
\Delta_{i,j,k} = \frac{1}{m(m-1)(m-2)}\cdot \left\{\begin{array}{ll}
n_in_k,& i=j;\\
0,& \text{otherwise}.\end{array}\right.
\end{equation}

Note that $\Delta$ can be written as $\frac{1}{m(m-1)(m-2)}\cdot A\otimes\vct n$,
where $A = \text{diag}(n_1,n_2,\cdots,n_V)$ is a $V\times V$ matrix and $\vct n = (n_1,n_2,\cdots,n_V)$ is defined previously.
As a result, $\Delta(W,W,W) = \frac{1}{m(m-1)(m-2)}\cdot (W^\top AW)\otimes \vct n$.
So the computational complexity of $\Delta(W,W,W)$ depends on how we compute $W^\top AW$.
Since $A$ is a diagonal matrix with at most $M$ non-zero entries, $W^\top AW$ can be computed in $O(Mk^2)$ operations.
Therefore, the time complexity of computing $\Delta(W,W,W)$ is $O(Mk^2)$ per document.

Finally we handle those values with three indices the same, that is, $i=j=k$.
As indicated by Eq. (\ref{eq_T}), we need to add a $\frac{2n_i}{m(m-1)(m-2)}$ term for compensation.
This can be done efficiently by first computing $\widehat{\mathbb E}(\frac{2n_i}{m(m-1)(m-2)})$
for all the documents (requiring $O(NV)$ time),
and then add them up, which takes $O(Vk^3)$ operations.